\documentclass[10pt,journal,compsoc]{IEEEtran}
\ifCLASSOPTIONcompsoc
\usepackage[nocompress]{cite}
\else
\usepackage{cite}
\fi

\usepackage{epsfig}
\usepackage{graphicx}
\usepackage{amsmath}
\usepackage{amssymb,amsthm,amsmath}
\usepackage{bm}
\usepackage{booktabs} 
\usepackage[caption=false]{subfig}
\usepackage{makecell}
\usepackage{multirow}
\usepackage{diagbox}

\newcommand{\LS}{\text{Location Subset}}

\newcommand{\Ehat}{\hat{E}}
\newcommand{\dec}{\textrm{dec}}
\newcommand{\decbar}{\overline{\dec}}

\newcommand{\MSE}{\textrm{MSE}}

\newcommand{\K}{\mathcal{K}}
\newcommand{\Khat}{\hat{\K}}

\newcommand{\Itrain}{I_{\textrm{tr}}}

\newcommand{\Jtrain}{J_{\textrm{tr}}}

\newcommand{\etal}{\textit{et al.}}
\newcommand{\D}{\mathcal{D}}
\newcommand{\Dtrain}{\mathcal{D}_\textrm{tr}}

\DeclareMathOperator{\Tr}{tr}

\usepackage[pagebackref=true,breaklinks=true,colorlinks,bookmarks=false]{hyperref}

\begin{document}
	\title{Towards High Performance Low Complexity Calibration in Appearance Based Gaze Estimation}

	\author{Zhaokang Chen and
		Bertram E. Shi,~\IEEEmembership{Fellow,~IEEE}
		\IEEEcompsocitemizethanks{\IEEEcompsocthanksitem Z. Chen and B.E. Shi are with the Department of Electronic and Computer Engineering, The Hong Kong University of Science and Technology, Clear Water Bay, Hong Kong. This work was supported in part by the Hong Kong Research Grants Council under grant number 16213617. \protect\\
			E-mail: zchenbc@connect.ust.hk, eebert@ust.hk}
		\thanks{}}

	\markboth{}%
	{Z. Chen and B.E. Shi: Dilation and Decomposition for Gaze Estimation}

	\IEEEtitleabstractindextext{%
		\begin{abstract}
			Appearance-based gaze estimation from RGB images provides relatively unconstrained gaze tracking from commonly available hardware. The accuracy of subject-independent models is limited partly by small intra-subject and large inter-subject variations in appearance, and partly by a latent subject-dependent bias. To improve estimation accuracy, we have previously proposed a gaze decomposition method that decomposes the gaze angle into the sum of a subject-independent gaze estimate from the image and a subject-dependent bias. Estimating the bias from images outperforms previously proposed calibration algorithms, unless the amount of calibration data is prohibitively large. This paper extends that work with a more complete characterization of the interplay between the complexity of the calibration dataset and estimation accuracy. In particular, we analyze the effect of the number of gaze targets, the number of images used per gaze target and the number of head positions in calibration data using a new NISLGaze dataset, which is well suited for analyzing these effects as it includes more diversity in head positions and orientations for each subject than other datasets. A better understanding of these factors enables low complexity high performance calibration. Our results indicate that using only a single gaze target and single head position is sufficient to achieve high quality calibration. However, it is useful to include variability in head orientation as the subject is gazing at the target. Our proposed estimator based on these studies (GEDDNet) outperforms state-of-the-art methods by more than $6.3\%$. One of the surprising findings of our work is that the same estimator yields the best performance both with and without calibration. This is convenient, as the estimator works well "straight out of the box," but can be improved if needed by calibration. However, this seems to violate the conventional wisdom that train and test conditions must be matched. To better understand the reasons, we provide a new theoretical analysis that specifies the conditions under which this can be expected. The dataset is available at \href{http://nislgaze.ust.hk}{http://nislgaze.ust.hk}. Source code is available at \href{https://github.com/HKUST-NISL/GEDDnet}{https://github.com/HKUST-NISL/GEDDnet}.
		\end{abstract}

		}

	\maketitle

	\IEEEpeerreviewmaketitle

	\IEEEraisesectionheading{\section{Introduction}\label{sec:introduction}}
	\IEEEPARstart{E}{ye} gaze has been used successfully in many promising applications, such as human-computer interfaces~\cite{menges2017gazetheweb,pi2017probabilistic}, human-robot interaction~\cite{huang2016anticipatory}, health care~\cite{grillini2018towards}, virtual reality~\cite{outram2018anyorbit,patney2016towards} and social behavioral analysis~\cite{hoppe2018eye,rogers2018using}. These successful applications have attracted more and more attention to eye tracking.

	To date, most commercial eye trackers have relied upon active illumination, e.g. pupil center corneal reflections (PCCR) based eye tracker used infrared illumination. While these eye trackers provide high accuracy, they are costly, commonly used in indoor laboratory settings, and place strong constraints on users' head movements. To alleviate the constraints on head movement, researchers have proposed many novel methods~\cite{brau2018multiple,chong2018connecting,fuhl2018cbf,wang2018slam,pi2019task}. Unfortunately, these eye trackers typically fail outdoors due to strong interference from the sun's infrared radiation.

	Appearance-based gaze estimation estimates gaze direction from RGB images. It provides relatively unconstrained eye tracking and can be used both indoors and outdoors. It requires only commonly available off-the-shelf cameras. Cameras are now ubiquitous features on electronic devices, which will enable the widespread deployment of gaze tracking without additional cost on many platforms, such as mobile devices~\cite{valliappan2020accelerating}.

	High accuracy appearance-based gaze estimation is challenging due to the variability caused by factors such as changes in appearance, illumination and head pose~\cite{zhang2019mpiigaze}. The application of deep convolutional neural networks (CNNs) has reduced estimation error significantly~\cite{zhang2015appearance}. Trained using a large number of high quality real and synthetic datasets
	covering a wide range of these variations
	\cite{fischer2018rt,funes2014eyediap,he2019photo,krafka2016eye,shrivastava2017learning,smith2013gaze,sugano2014learning,wang_2018_CVPR,wood2016learning,zhang2015appearance}, deep CNNs can learn to compensate for much of the variability \cite{chen2018appearance,cheng2018appearance,deng2017monocular,krafka2016eye,lian2018multiview,palmero2018recurrent,parekh2017eye,ranjan2018light,xiong2019mixed,zhang2017s}. Despite these improvements, the estimation error of appearance-based approaches ($\sim$$5^\circ$) is still not as low as that achieved by using active illumination ($\sim$$1^\circ$).

	One way to reduce estimation error is to introduce subject-dependent calibration. For example, PCCR-based eye trackers typically require an initial calibration before they are used. The user is usually required to gaze sequentially at nine targets on a $3\times 3$ grid~\cite{guestrin2006general}. The data collected is used to calibrate a geometrical model of the eye. However, calibration for appearance-based methods has received relatively little attention to date.

    Simpler calibration procedures are desirable. However, there is a trade-off between simplicity and accuracy. The challenge we address in this paper is how to achieve the best accuracy with the lowest complexity. There are several dimensions along which we measure calibration set complexity: the number of different gaze targets, the number of images per gaze target, and the number of head positions at which data is collected.

    Reducing the number of gaze targets is highly desirable. First, it is less time consuming. For example, the nine gaze points used in PCCR calibration are sometimes reduced to five points arranged as a cross to save time. Second, calibration imposes additional constraints. Requiring the user to gaze at a number of known locations requires a mechanism, e.g. a screen, for specifying those points. Ideally, one could reduce the number of gaze targets to one: single gaze target calibration. In this case, a very convenient gaze target would be the camera, since this must always be visible to the user.

	Previously, we have described a calibration technique that decomposes gaze estimates into the sum of a gaze estimate that is estimated by a subject-independent deep dilated CNN and a subject-dependent bias~\cite{chen2020offset}. This was based on our observation that a subject dependent additive offset is the dominant factor that degrades accuracy. We demonstrated that gaze decomposition (GD) achieves better accuracy with low complexity calibration sets. While other more complex algorithms can eventually outperform ours, this only occurs for prohibitively large calibration datasets. We recap those results here for completeness.

	Reducing the number of images per gaze target reduces the amount of time the user spends gazing at each target. On the other hand, more images enables us to capture more variability in the head orientation, which improves estimates of the calibration parameters. While we have investigated the effect of increasing the number of images of a subject previously, this was done using datasets with limited head pose variability and for subjects looking at targets that were close, but not identical. Here we present a more detailed analysis using a new dataset introduced here, NISLGaze. This dataset has increased variability in head pose collected during each viewing of each target, enabling us to investigate the effect of head orientation variability on the estimation of calibration parameters.

	Most calibration procedures for eye tracking are performed with the head at only one position. However, the accuracy of PCCR eye trackers is very sensitive to changes in head position. If accurate eye-tracking is to be maintained, re-calibration as the head moves is required~\cite{pi2019task}. One hope for appearance based gaze estimation is that its estimates will be more robust to head location, e.g. due to the application of techniques such as normalization~\cite{zhang2018revisiting}. However, to our knowledge, robustness to head position has not been evaluated before. The NISLGaze dataset contains images of subjects at different head positions within the camera images, enabling for the first time the systematic evaluation of this effect. Using this dataset, we investigate two separate, yet related issues: robustness of the gaze estimates to head position, and robustness of the calibration parameters.

	Finally, we address the issue of making calibration optional. Due to the overhead it introduces, it may be desirable to forgo it in some cases, e.g. when users would have only short interaction with a gaze-aware human-computer or human-robot interface, yet still have the option for calibration if users desired more accurate estimates. Conventional wisdom would suggest we would need two separate estimators: one trained with calibration and one trained without calibration to avoid train/test mismatch and achieve the best performance in each scenario. Surprisingly, we find this is not necessary. Despite the train-test mismatch, applying gaze decomposition during training actually improves performance of the gaze estimator even when calibration is not used. This paper provides a new theoretical analysis, elucidating the conditions under which this effect occurs.

	Our experiments here are based mainly around a previously proposed deep CNN network for gaze estimation, which replaces max-pooling in the higher convolutional layers with dilated convolutions~\cite{chen2018appearance}. We have demonstrated that this replacement improved the performance of a gaze estimator built around a VGG backbone. Here, we provide additional support for the advantages of this replacement. We show that replacing max-pooling/large stride by dilated convolutions also improve the performance of more modern networks, e.g. those based on ResNet. We provide a sensitivity analysis showing that representations computed by networks with dilated convolutions are more responsive to the small changes in appearance due to gaze shifts.

	\section{Related Work}

	\subsection{Appearance-Based Gaze Estimation}
	Past approaches to appearance-based gaze estimation have included Random Forests \cite{sugano2014learning}, k-Nearest Neighbors \cite{sugano2014learning,schneider2014manifold}, and Support Vector Regression \cite{schneider2014manifold}. More recently, the use of deep CNNs to appearance-based gaze estimation has received increasing attention. Zhang \textit{et al.} proposed the first deep CNN for gaze estimation in the wild \cite{zhang2015appearance,zhang2019mpiigaze}, which gave a significant improvement in accuracy.

	Considering regions of the face outside the eyes has further improved accuracy. For example, Krafka \textit{et al.} proposed a CNN with multi-region input, including an image of the face, images of both eyes and a face grid~\cite{krafka2016eye}. Parekh \textit{et al.} proposed a similar multi-region network for eye contact detection~\cite{parekh2017eye}. Zhang \textit{et al.} proposed to take the full face image as input and adopts a spatial weighting method to emphasize features from particular regions of the face~\cite{zhang2017s}.

	Other work has focused on how to extract better information from eye images. Cheng \textit{et al.} developed a multi-stream network to address the asymmetry in gaze estimation between left and right eyes~\cite{cheng2018appearance}. Yu \textit{et al.} proposed to estimate the locations of eye landmark and gaze directions jointly~\cite{yu2018deep}. Park \textit{et al.} proposed to learn an intermediate pictorial representation of the eyes~\cite{park2018deep}. Lian \textit{et al.} proposed to use images from multiple cameras~\cite{lian2018multiview}. Xiong \etal{} proposed mixed effects neural networks to tackle the problem caused by the non \textit{i.i.d.} natural of eye tracking datasets~\cite{xiong2019mixed}.

	Still others have sought to improve robustness to the change in head pose. Deng and Zhu estimated the head pose in camera-centric coordinates and the gaze directions in head-centric coordinates, and then combined them geometrically~\cite{deng2017monocular}. Ranjan \textit{et al.} applied a branching architecture, whose parameters were switched according to a clustering of head pose angles~\cite{ranjan2018light}. Palmero \etal{} proposed to use 3D shape cue of the face~\cite{palmero2018recurrent}.

	The deep CNN network that is the basis of the gaze estimator we propose in this paper adopts dilated-convolutions, rather than max-pooling\cite{chen2018appearance}. While dilated-convolutions have been widely used for pixel-wise tasks, such as segmentation~\cite{yu2015multi,chen2017rethinking}, they have rarely been used for more global tasks, like regression or classification. We show that by learning high-level features at high resolution, dilated-convolutions better capture the subtle changes in eye appearance as gaze changes.

	\subsection{Calibration in Appearance-Based Gaze Estimation}

	To the best of our knowledge, the first implementation of personal calibration was in iTracker~\cite{krafka2016eye}. After training a subject-independent network, this system was calibrated by training a subject-dependent Support Vector Regressor on the features extracted from the last fully-connected layer of the subject-independent network. It used images collected as the subject gazed at 13 different locations as the calibration set. Error was reduced by about $20\%$ when calibrated on the full set, but increased when only calibrated on a subset of 4 locations, most likely due to overfitting. Lind{\'e}n \textit{et al.} proposed a network that had some subject-dependent latent parameters~\cite{linden2018appearance}. During training, these parameters were learned from the training data. During testing, these parameters were re-estimated by minimizing the estimation error over a calibration set. The calibration set required 45 to 100 images collected as the subject gazed at multiple points. Liu \textit{et al.} learned a homogeneous linear transformation matrix to warp the estimates from a subject-independent estimator~\cite{liu2018differential}. Zhang \etal{} learned a third-order polynomial function to warp the estimates~\cite{zhang2019evaluation}.

	The past work described above assumed high complexity calibration sets. Only recently has attention been paid to reducing the complexity of the calibration set. Liu \textit{et al.} proposed a differential approach for calibration, where they trained a subject-independent Siamese network to estimate the difference between the gaze angles of two images of the same subject~\cite{liu2018differential,liu2019differential}. Yu \etal{} proposed gaze redirection synthesis, which augment the calibration samples using a generative adversarial network (GAN) to fine-tune a subject-independent estimator~\cite{yu2019improving}. Park \etal{} proposed FAZE, which trained the last multi-layer perceptron using meta-learning to learn good initial weights that can adapt to a few samples~\cite{park2019few}. These works only require a few calibration samples gazing at different gaze targets.

	In comparison to the work described in the previous two paragraphs, we only adapt the subject-specific bias during calibration. The gaze decomposition is similar in spirit to the person-dependent latent parameters proposed by Lind{\'e}n et al.~\cite{linden2018appearance}, but there are several important differences. First, we use much fewer parameters (two v.s. six). Second, the parameters are introduced in the output layer, rather than fed to the last fully-connected layers. This results in far superior performance when only few calibration images are available. Reducing the size and complexity of the calibration set makes calibration easier and less time-consuming.

	\subsection{Eye Tracking Datasets}
	\begin{figure}[!t]
		\centering
		\subfloat[$(10^\circ \rm H, 0^\circ \rm V)$]{
			\includegraphics[width=2.6cm]{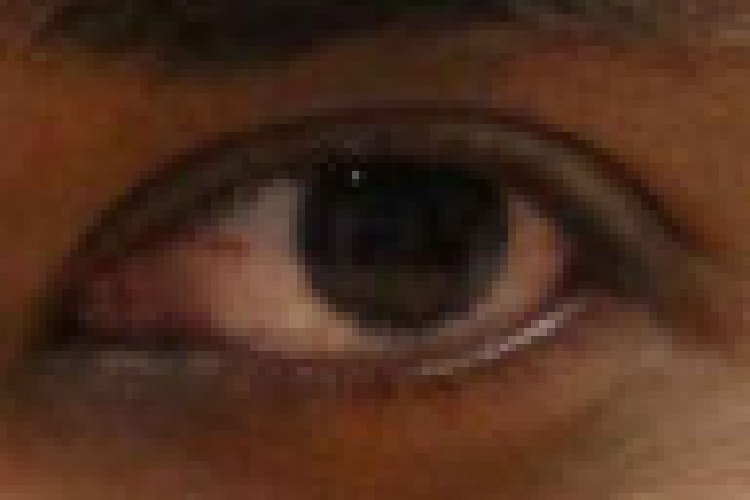}
			\label{fig:eyeDiffa}
		}
		\subfloat[$(15^\circ \rm H, 0^\circ \rm V)$]{
			\includegraphics[width=2.6cm]{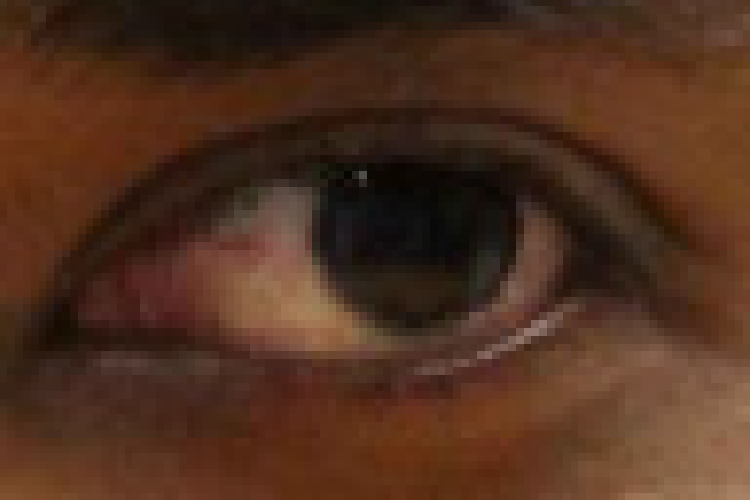}
			\label{fig:eyeDiffb}
		}
		\subfloat[Differences]{
			\includegraphics[width=2.6cm]{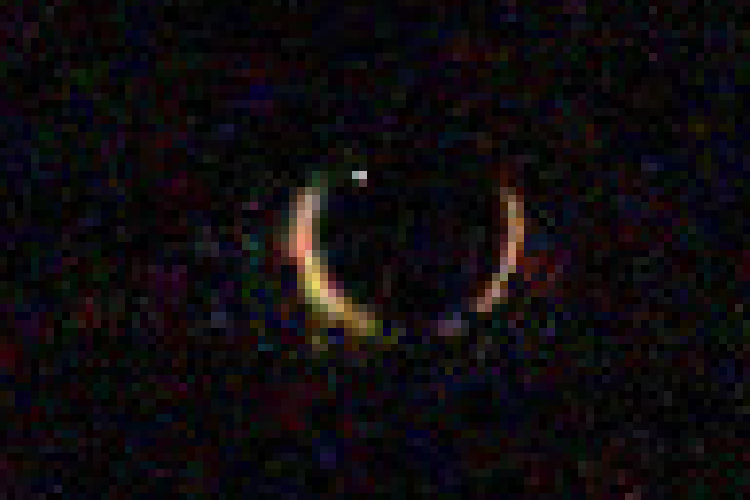}
			\label{fig:eyeDiffc}
		}
		\caption{Images of two left eyes and their difference from Columbia Gaze dataset~\cite{smith2013gaze}. (a) Image with $10^\circ$ horizontal and $0^\circ$ vertical gaze angle. (b) Image with $15^\circ$ horizontal and $0^\circ$ vertical gaze angle. (c) The absolute difference between (a) and (b) (value is scaled for better illustration).}
		\label{fig:1}
	\end{figure}
	Available datasets for training appearance-based gaze estimators can be categorized as either real-world or synthetic based on the image generation process, and as either eye-only or full-face based on image content.

	Real-world datasets contain images of real people taken as they gaze at different points in the environment. ColumbiaGaze contains images of 56 subjects with a discrete set of head poses and gaze targets~\cite{smith2013gaze}. EYEDIAP contains videos of 16 subjects gazing at targets on a screen or floating in 3D~\cite{funes2014eyediap}. MPIIGaze contains images of 15 subjects when using their laptops~\cite{zhang2015appearance}. GazeCapture contains images of 1,474 subjects taken while they were using tablets~\cite{krafka2016eye}. RT-GENE contains images of 15 subjects behaving naturally~\cite{fischer2018rt}. Commercial eye tracking glasses were used to provide ground truth gaze direction, and a GAN was used to remove the eye tracking glasses from the images.

	To provide samples with larger variations in head pose and appearance, some researchers have generated datasets of synthetic images. The UT-multiview~\cite{sugano2014learning} and UnityEyes~\cite{wood2016learning} datasets used real-world images and 3D eye region models to render eye-only images at arbitrary head poses and gaze directions. Recently, GANs have been used to generate photorealistic samples~\cite{shrivastava2017learning,wang_2018_CVPR,he2019photo}.

	We describe here a new dataset, NISLGaze, which contains full-face images of 21 subjects. NISLGaze contains larger variability in face locations and head poses than existing datasets. This variability makes it useful for estimating the performance of appearance-based gaze estimators in relatively unconstrained settings.

	Recently, Gaze360~\cite{kellnhofer2019gaze360}, ETH-XGaze~\cite{zhang2020eth} and a dataset collected by Deng and Zhu ~\cite{deng2017monocular} have been reported. These provide images of more than 100 subjects with large variations in head orientation. However, most faces in ETH-XGaze and the Deng/Zhu dataset are located at the center of the images. While faces in Gaze360 appear at more diverse image locations, there is no guarantee that one subject will appear at multiple locations. Although the NISLGaze dataset has fewer subjects, it collects more data per subject, and ensures that each subject appears at locations widely distributed throughout the images. Therefore, we can use this dataset to systematically evaluate our proposed network and subject-dependent calibration method at multiple locations and orientations.

	\section{METHODOLOGY}
	\subsection{Spatial Resolution}
	Image differences due to gaze shifts can be subtle, especially if the head remains fixed. As shown in Fig.~\ref{fig:1}, eye images with gaze angles differing by $5^\circ$ horizontally differ at only a few pixels. Intuitively, extracting high-level features at high resolution can better capture these subtle differences. However, most current CNN architectures use multiple downsampling layers, e.g. convolutional layers with large strides and max-pooling layers, which reduce spatial resolution. This is useful for classification, as the network tolerates small variations in position. However, for gaze estimation, valuable information is lost.

	To address this problem, we have proposed to replace convolutional/downsampling layers by dilated-convolutional layers. Dilated-convolutions increase receptive field (RF) size while keeping the number of parameters manageable by inserting spaces (zeros) between weights. Let $(x, y)$ and $(n, m)$ represent spatial position, and $k$ indicate the feature map. Given an input feature map $u(x, y, k)$, a weight kernel $w_{nmk}$ of size $N\times M\times K$, a bias term $b$, and dilation rates $(r_1, r_2)$, the output feature map $z(x, y)$ of a dilated-convolutional operation can be calculated by
	\begin{equation}
	z(x,y)=\sum_{k=1}^{K}\sum_{m=0}^{M-1}\sum_{n=0}^{N-1}u(x+nr_1,y+mr_2,k)w_{nmk}+b.
	\label{dilated}
	\end{equation}
	The dilation rates $(r_1, r_2)$ determine the amount by which the size of RF increases.

	\subsection{Subject-Dependent Bias}
	\label{sec:bias}

	\begin{figure}
		\centering
		\includegraphics[width=8.5cm,height=2.2cm]{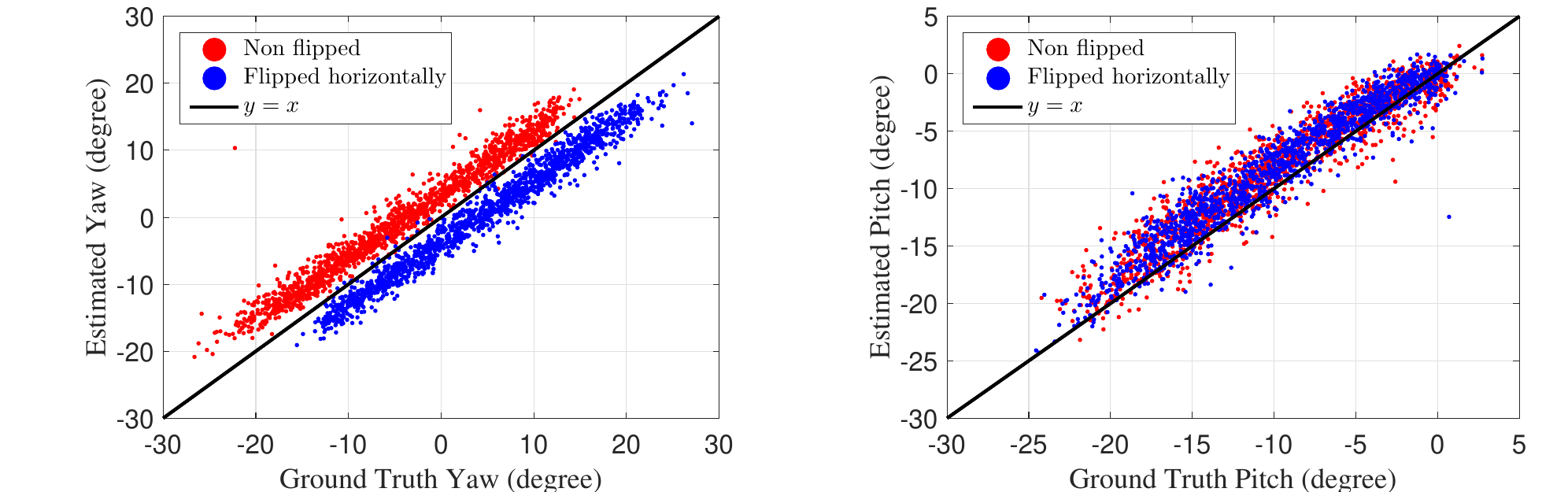}
		\caption{Scatter plots of the estimated gaze angles versus the ground truth of subject p06 from MPIIGaze.}
		\label{fig:bias_subj}
	\end{figure}
	\begin{figure}
		\centering
		\subfloat[]{
			\includegraphics[width=2.8cm]{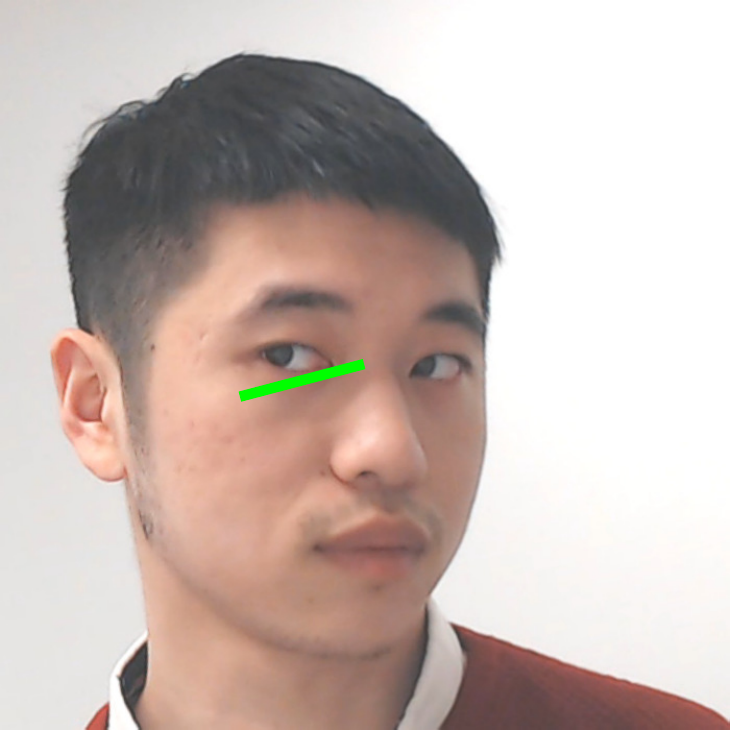}
			\label{fig:overview1}
		}
		\subfloat[]{
			\includegraphics[width=4.5cm]{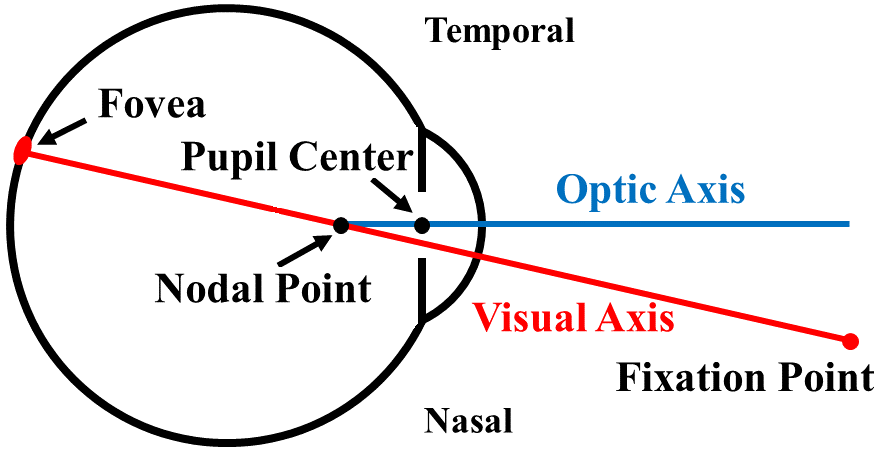}
			\label{fig:overview2}
		}
		\caption{Overview. (a) Appearance-based gaze estimation estimates the gaze angles (pitch and yaw) from RGB images. However, there are some factors not observable from the images, such as (b) the subject-dependent offset between the optic axis and visual axis.}
		\label{fig:overview}
	\end{figure}
	Fig.~\ref{fig:bias_subj} compares the gaze angles estimated by a subject-independent estimator~\cite{chen2018appearance} and the ground truth for one subject (p06) from MPIIGaze~\cite{zhang2015appearance}. The difference between an estimated angle and the corresponding ground truth is presented by the vertical difference between a color dot and the black line $y=x$. The scatter plots indicate that there is a bias between the estimates and ground truth angles, which is quite constant across gaze angles. When the image are flipped horizontally, the yaw bias changes sign, but the pitch bias stays similar.

	The bias has two sources: first, there are components of gaze that vary from subject to subject that cannot be estimated from appearance. For example, as shown in Fig.~\ref{fig:overview}(b), there is a deviation between the visual axis (the line connecting the nodal point with the fovea) and the optic axis (the line connecting the nodal point with the pupil center) of an eye, which varies from person to person~\cite{atchison2000optics, guestrin2006general}. Second, there may be changes in appearance between subjects that are correlated with changes due to gaze shifts.

	Fig.~\ref{fig:Error} analyzes the error of the subject-independent estimator~\cite{chen2018appearance} across subjects on MPIIGaze in leave-one-subject-out cross-validation. Fig.~\ref{fig:Error} shows the means and standard deviations (SDs) of the yaw and pitch error, for different subjects and for both the original and horizontally flipped images. These results show that in both yaw and pitch angles, the errors are generally biased. The bias varies widely across subjects. When a image is horizontally flipped, the yaw bias typically has similar magnitude but a different sign, while the pitch bias remains similar. The mean squared bias across subjects (16.2 $\mathrm{deg}^2$) exceeds the mean intra-subject variance (12.9 $\mathrm{deg}^2$), indicating that the bias is a significant contributor to the error.

	Motivated by these findings, for the $j^{\text{th}}$ image of the $i^{\text{th}}$ subject in a dataset, $X_{i,j}$, we decompose our gaze estimate $\hat{g}(X_{i,j})\in\mathbb{R}^2$ in (yaw, pitch) into the sum of a subject-independent estimate $\hat{t}(X_{i,j};\Phi)\in\mathbb{R}^2$ and a subject-dependent bias $\hat{b}_i\in\mathbb{R}^2$:
	\begin{equation}
	\hat{g}(X_{i,j})=\hat{t}(X_{i,j};\Phi)+\hat{b}_i,
	\label{dec}
	\end{equation}
	where $\Phi$ denotes the parameters of the dilated CNN described below. During training, the subject-dependent bias is estimated by minimizing estimation error. During testing/deployment, it is estimated from calibration data or set to zero in the absence of training data.

	\begin{figure}[!t]
		\centering
		\subfloat[Mean and SD of yaw errors across 15 subjects.]{
			{\includegraphics[width=0.95\columnwidth]{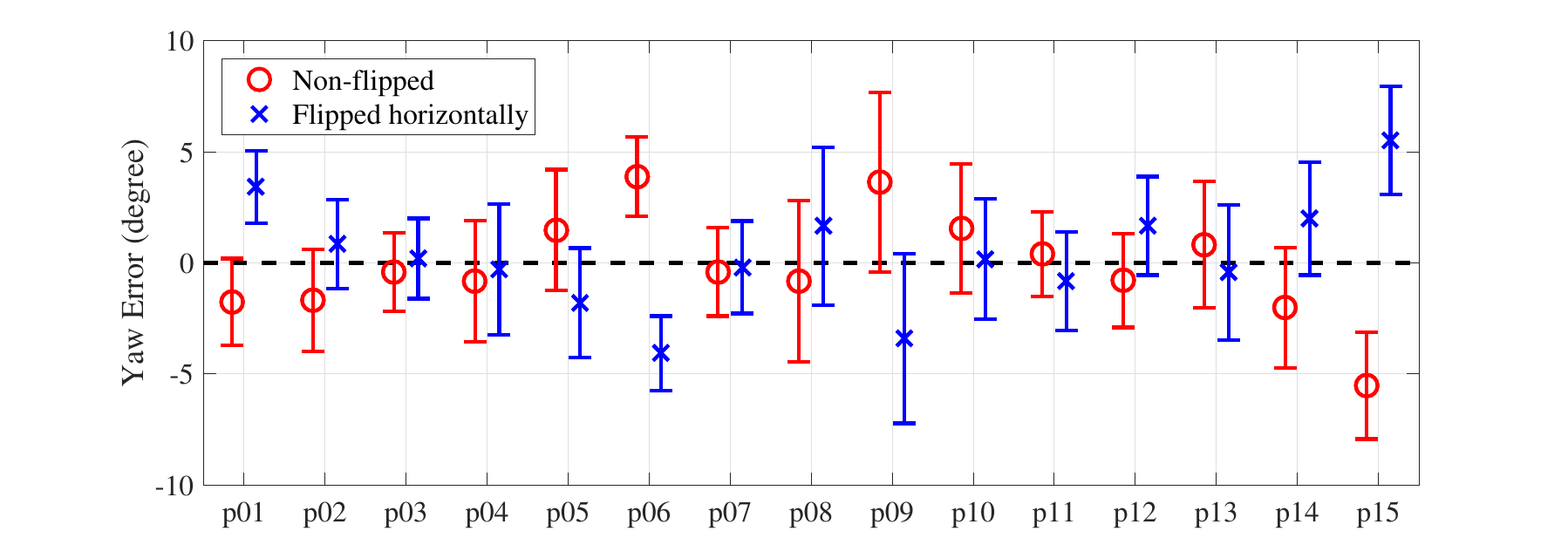}}
			\label{fig:biasa}
		}

		\subfloat[Mean and SD of pitch errors across 15 subjects.]{
			{\includegraphics[width=0.95\columnwidth]{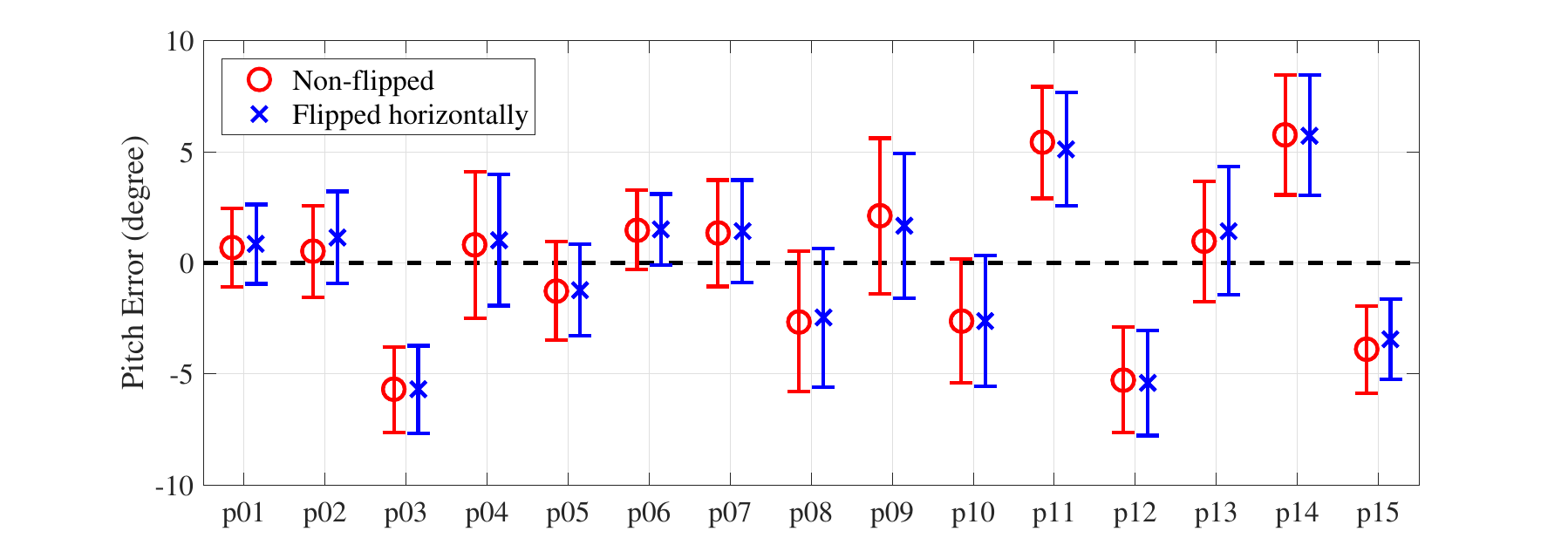}}
			\label{fig:biasb}
		}
		\caption{Error analysis of our subject-independent estimator~\cite{chen2018appearance} on the MPIIGaze dataset~\cite{zhang2015appearance} for (a) yaw and (b) pitch. Circles and crosses indicate the means. Error bars indicate standard deviations.}
		\label{fig:Error}
	\end{figure}
	\subsection{Preprocessing}
	We adopt the data normalization method in~\cite{zhang2018revisiting}. This method first applies a virtual three degree of freedom (DoF) rotation to the camera. Two DoF are used to point the virtual camera towards a reference point on the face. One DoF cancels the head roll so that the lines connecting the midpoints of the eyes are horizontal. The method then applies a scaling so that images appear as if they are taken from a fixed distance. The normalized images are converted to gray scale and histogram-equalized. The ground truth gaze angles are normalized correspondingly.

	As stated in~\cite{zhang2018revisiting}, the normalized images still have two degrees of freedom in head pose (pitch and yaw). These must be taken into account by the deep-neural-network-based gaze estimator.

	 We align the images based on detected landmarks from dlib~\cite{king2009dlib} and OpenFace~\cite{baltrusaitis2018openface}. The face images are aligned so that the center of each eye is at a fixed location in the resulting images. The eye images are aligned so that the eye corners are at fixed locations.

	\subsection{Architecture}
	\begin{figure*}
		\centering
		\includegraphics[width=15.5cm]{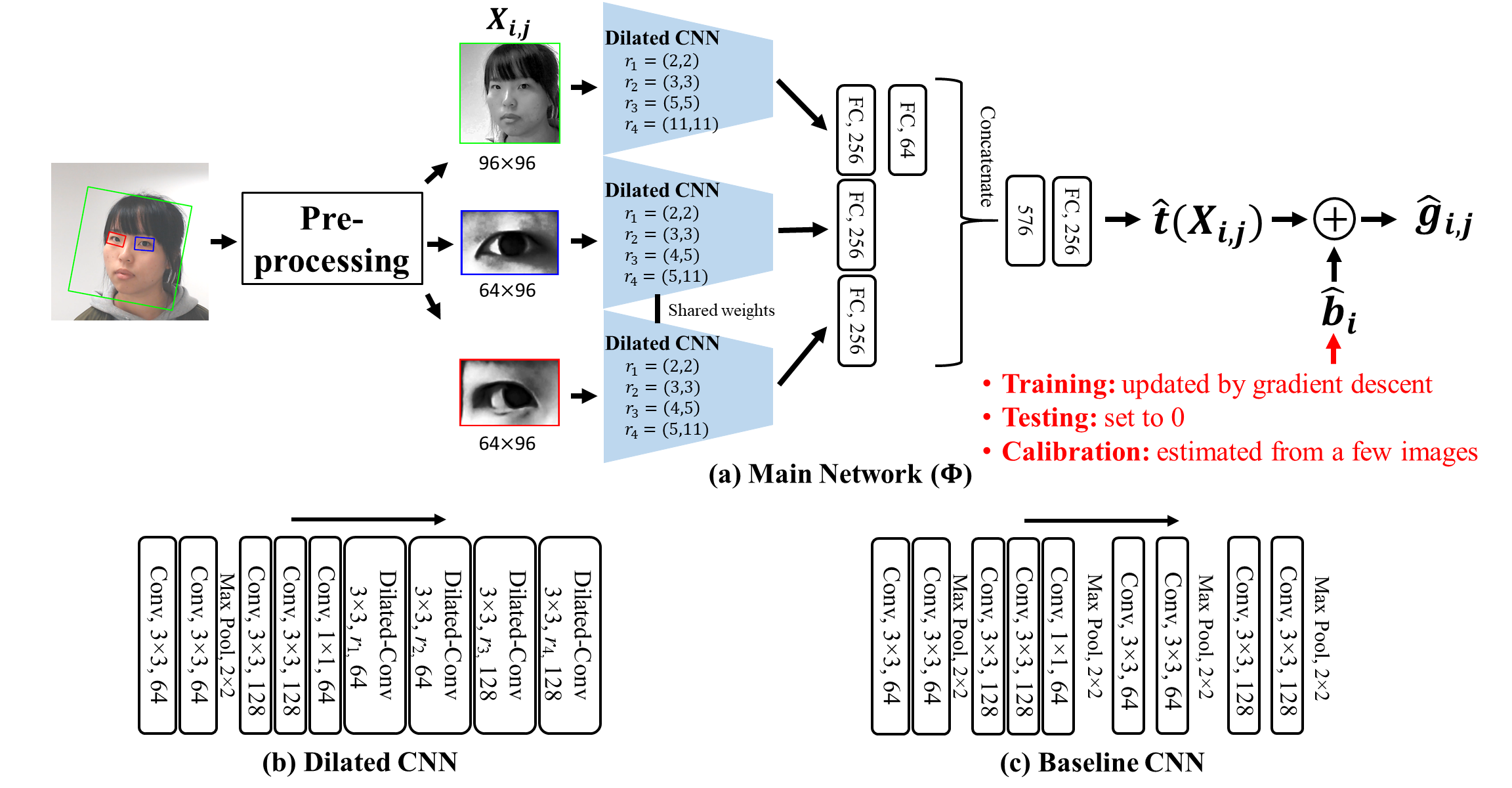}
		\caption{Architecture of the proposed network. (a) The main network that outputs $\hat{t}(X_{i,j})$ based on the input images $X_{i,j}$. (b) The dilated CNN that is the basic component of (a). (c) The regular CNN without dilated-convolutions used as a baseline. FC denotes fully-connected layers, Conv denotes convolutional layers, and Dilated-Conv denotes dilated-convolutional layers with $r$ as the dilation rate.}
		\label{fig:Archi}
	\end{figure*}

	We propose GEDDnet, a network for Gaze Estimation with Dilation and (gaze) Decomposition. The GEDDnet architecture is presented in Fig.~\ref{fig:Archi}. It is inspired by iTracker~\cite{krafka2016eye}, but differs in that we do not use the face grid as input, and in that we replace some convolutional and max-pooling layers with dilated-convolutional layers.

	Our network takes images of the face and of both eyes as input. The input images, collectively referred to by $X_{i,j}$, are first fed separately to three dilated CNNs. The two dilated CNNs that take the eyes as input share the same weights. The feature extracted by the dilated CNNs are then connected to fully-connected (FC) layers, concatenated, fed to another FC layer, and finally to a linear output layer. We denote the parameters of this network by $\Phi$.

	The architecture of the dilated CNN is shown in Fig.~\ref{fig:Archi}(b). It has one max-pooling layer, five regular convolutional layers and four dilated-convolutional layers with different dilation rates. The strides for all convolutional layers are 1.
	We used dilated CNNs only in the upper layers, as this enabled us to  exploit transfer learning for the lower layers, which were initialized from an ImageNet-pretrained model. Transfer learning is appealing, since gaze datasets have much fewer samples than large-scale image classification datasets like ImageNet. Lower-level features generalize well to different tasks. For example, we have found that networks with all dilated layers or with the same structure but trained from random initial weights achieve errors without calibration of $5.7^\circ$ and $5.2^\circ$, compared to the $4.5^\circ$ for GEDDNet reported in Table~\ref{table:calFree}.

	We use Rectified Linear Units (ReLUs) as the activation functions. We apply zero-padding to regular convolutional layers, but no padding to dilated-convolutional layers. The initial weights of the first four convolutional layers are transferred from VGG-16~\cite{simonyan2014very} pre-trained on the ImageNet dataset~\cite{deng2009imagenet}. The weights in all other layers are randomly initialized. Batch renormalization layers~\cite{ioffe2017batch} are applied to all layers trained from scratch. Dropout layers with dropout rates of 0.5 are applied to all FC layers.

	\subsubsection{Training}
	We train the network by solving the following problem by mini-batch gradient descent:
	\begin{equation}
	\min_{\Phi,\hat{b}_i} \left(\sum_{i,j}\begin{Vmatrix}(g_{i,j}-\hat{t}(X_{i,j};\Phi)-\hat{b}_i\end{Vmatrix}_2^2
	+ \lambda \begin{vmatrix}\sum_i \hat{b}_i\end{vmatrix}\right),
	\label{train_t}
	\end{equation}
    where $\begin{Vmatrix}\cdot\end{Vmatrix}_2$ represents the Euclidean norm,  $\begin{vmatrix}\cdot\end{vmatrix}$ represents the absolute value, and $\hat{b}_i$ are the learned estimates of the subject-dependent bias. The second term is a regularizer that ensures that the mean subject-dependent bias over the training set is zero. It is not a sparsity regularizer, which would have the form $\sum_i |\hat{b}_i|$. Training is insensitive to the value of $\lambda$ (we used $\lambda=1$) and to the norm used, because the second term can be easily optimized to zero without changing the value of the first term, since every change of the mean bias can be compensated by an equivalent change in the constant offset of the last layer, which computes $\hat{t}$. Rather than treating the biases $\hat{b}_i$ as extra learned parameters, they could be estimated by the mean offset for each subject over the training set. However, this would be time-consuming given the large size of the training set, especially since we apply online data augmentation. Removing the last layer bias would not remove the need for the regularizer, since it is possible that one of the neurons feeding the last layer could serve as a constant offset. The regularizer could be removed if we set the bias of one of the subjects in the training set to zero. However this would require us to set $\hat{b}_m$ equal to the mean bias across the remaining subjects in the training set if no calibration images are available. This adds an extra parameter required during deployment, which would vary depending on the training set and the subject chosen to have zero bias. Thus, we prefer to add the regularizer during training to simplify deployment, but both approaches would result in similar performance.

	We implement our network in TensorFlow. We use the Adam optimizer with default parameters. The batch size is 64. The initial learning rate is 0.001, which is divided by 10 after every ten epochs. The training proceeds for 35 epochs. We apply standard online data augmentation techniques, including random cropping, scaling, rotation and horizontal flipping. Since the bias varies if the images are horizontally flipped, we consider the non-flipped and flipped images as belonging to different subjects.

	\subsubsection{Testing and Calibration}
	During testing, gaze estimates are obtained according to Eq.~\eqref{dec}. For a new subject~$m$, we set $\hat{b}_m = 0$ if no calibration images are available.

	A calibration set $\D_m$ contains image-gaze pairs for a subject~$m$, i.e., $\D_m=\{(X_{m,j}, g_{m,j}),j=1,2,\dots,|\D_m|\}$, where $|\D_m|$ denotes the cardinality of $\D_m$. We measure the complexity of the calibration set by the number of gaze targets $T$, the number of images per gaze target $S$ and the number of head positions $P$. Thus, $S\cdot T \cdot P=|\D_m|$. For best performance, the images for each gaze target should capture the variability experienced during testing, e.g. in head pose and/or illumination.

	Given a calibration set $\D_m$, we set $\hat{b}_m$ equal to the Maximum A Posteriori (MAP) estimate of the bias. We assume that $g_{m,j}-\hat{t}(X_{m,j}) \sim \mathcal{N}({b}_m, \sigma_{\hat{t}}^2\rm{I})$ and that $b_m \sim \mathcal{N}(\bm{0}, \sigma_0^2\rm{I})$, where $\mathcal{N}$ represents the normal distribution, and that the two random variables are independent According to Bayes' rule,
	\begin{equation}
	\begin{aligned}
	\hat{b}_m&=\mathop{\arg\max}_{b_m} \Big(p(b_m)\prod_{j=1}^{|\D_m|}p(g_{m,j}-\hat{t}(X_{m,j})|b_m)\Big)\\
	&=\frac{1}{|\D_m|+\frac{\sigma_{\hat{t}}^2}{\sigma_0^2}}\sum_{j=1}^{|\D_m|}\Big(g_{m,j}-\hat{t}(X_{m,j})\Big),
	\label{cali}
	\end{aligned}
	\end{equation}
	where $p(\cdot)$ represents the probability density function. We estimate the values of $\sigma_0$ and $\sigma_{\hat{t}}$ from training. The calibration method proposed in our previous work~\cite{chen2020offset} is a special case of this method ($\sigma_{\hat{t}}^2=0$), corresponding to the Maximum Likelihood (ML) estimate.

	\section{Datasets}
	\subsection{Existing Datasets}
	\noindent\textbf{MPIIGaze}. This dataset contains images of full face of 15 subjects (six females, five with glasses). We use the ``Evaluation Subset'', which contains 3,000 images per subject. Half of the images are flipped horizontally in this subset. The reference point for image normalization is set to the center of the face for within-dataset evaluation as in~\cite{zhang2017s}, and to the midpoint of both eyes for cross-dataset evaluation for consistency with other datasets.

	\noindent\textbf{EYEDIAP}. This dataset contains videos of full face with continuous screen target, discrete screen target or floating target, and with static or dynamic head pose. We use the data from screen targets, which includes 14 subjects (three female, none with glasses). The reference point for image normalization is set to the midpoint of both eyes.

	\noindent\textbf{ColumbiaGaze}. This dataset contains 5,800 images of full face of 56 subjects (24 females, 21 with glasses). For each subject, images are collected for each combination of five horizontal head poses ($0^\circ$, $\pm 15^\circ$, $\pm 30^\circ$), seven horizontal gaze angles $H$ ($0^\circ$, $\pm 5^\circ$, $\pm 10^\circ$, $\pm 15^\circ$) and three vertical gaze angles $V$ ($0^\circ$, $\pm 10^\circ$). The reference point for image normalization was set to the midpoint of both eyes.

	\subsection{The NISLGaze Dataset}
	To estimate the performance of the proposed algorithm over the dimensions of interest we have identified, we collected the NISLGaze Dataset. This dataset contains 21 subjects (10 females and 10 with glasses). During collection, the camera was rotated so that each subject appeared in nine different image locations (see Fig.~\ref{fig:dataset2}(a)). For each location, the subjects gazed at nine different blocks on a wall (see Fig.~\ref{fig:dataset2}(b)) while moving their heads freely but keeping their body position fixed. Fig.~\ref{fig:dataset1} shows that the distributions of face locations ($\bm{l}$), head poses ($\bm{h}$) and gaze directions ($\bm{g}$) in the normalized space~\cite{zhang2018revisiting} of NISLGaze are wider than those of the commonly used MPIIGaze~\cite{zhang2015appearance} and EYEDIAP~\cite{funes2014eyediap} datasets. The experimental procedures involving human subjects described in this article were approved by Committee on Research Practices at the Hong Kong University of Science and Technology.

	\subsubsection{Data Collection}
	We used a Logitech C922 webcam to record videos at a resolution of $1920\times 1080$ at 30 fps. The subjects were instructed to stand 90 cm away from a white wall. On the wall, a 60 cm $\times$ 60 cm region was uniformly separated into $3\times 3$ blocks. Inside each block, except the central block, black crosses arranged in a $3\times 3$ grid were presented as gaze targets (see Fig.~\ref{fig:dataset2}(b)). The height of the center of the central block was 160 cm, which was about the eye height of the subjects. The camera was fixed at 10 cm away from the wall and centered in the central block.

	The data collection was divided into nine sessions. The face location was fixed during the whole collection. In each session, the face was at one of the nine locations shown in Fig.~\ref{fig:dataset2}(a) by rotating the camera. The subject was instructed to gaze at the nine gaze blocks sequentially in a random order. The central block was gazed at three times, while the others were gazed at only once. When the subject was asked to gaze at the central block, s/he was instructed to look at the camera; When s/he was asked to gaze at the other eight blocks, s/he was instructed to look at each of the crosses one by one. Transitions between gaze targets were signaled by a notification sound. The subject was instructed to rotate his/her head freely. We recorded one video for each gaze block, which lasted for about 30 s. This collection procedure yielded 2,079 videos (11 videos/location $\times$ 9 locations/subject $\times$ 21 subjects $=$ 2,079 videos. For each subject, data collection took from one to three days. For most subjects (15), data collection spanned three days (three sessions per day).
	\begin{figure}
		\centering
		\subfloat[Nine different face locations]{
			\includegraphics[width=5cm]{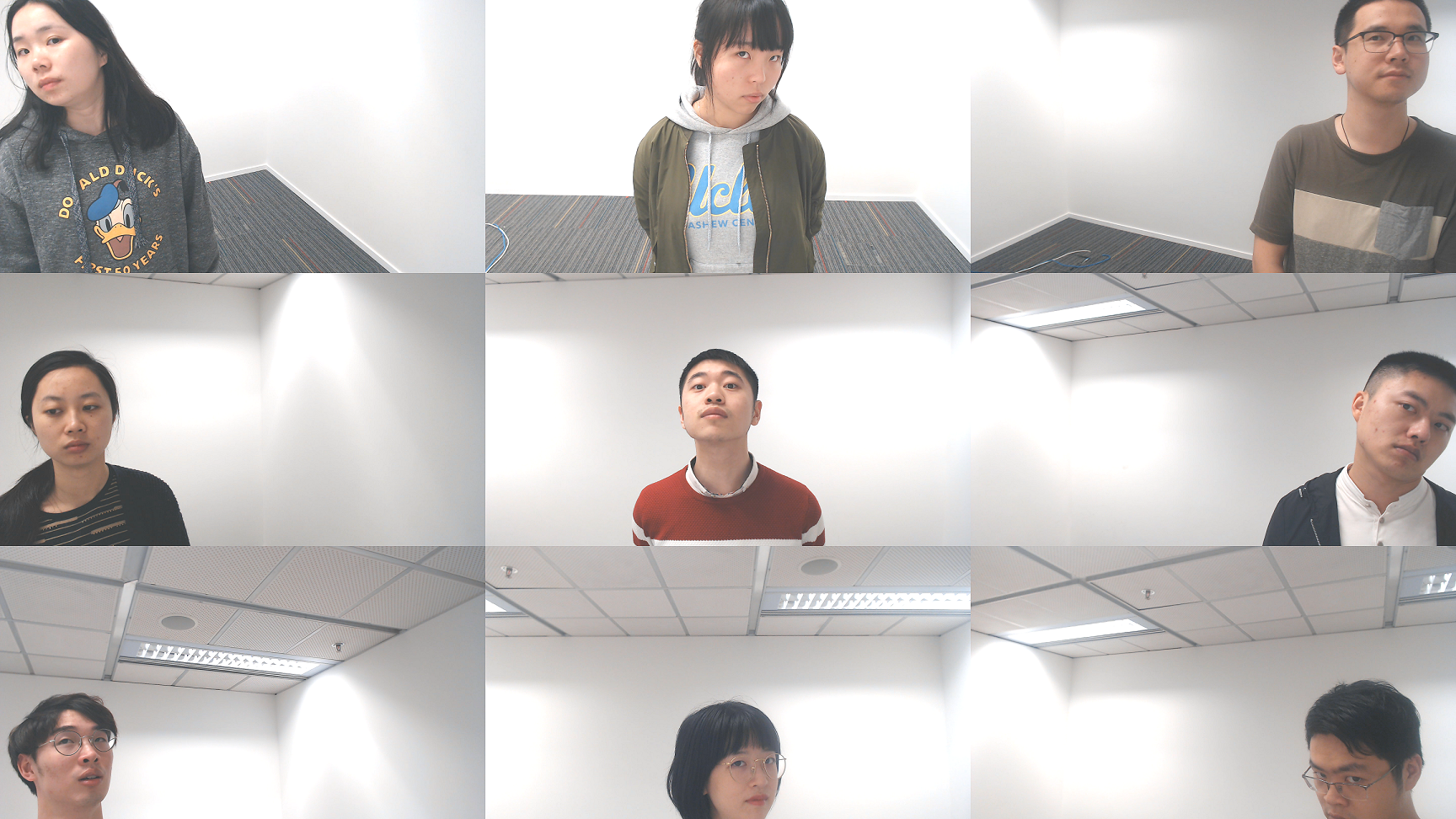}
			\label{fig:dataset2l}
		}
		\subfloat[Gaze targets]{
			\includegraphics[width=2.8cm]{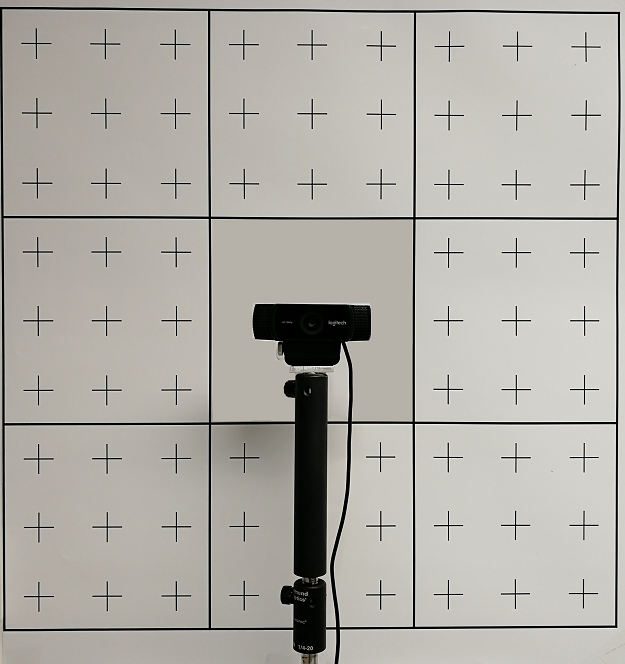}
			\label{fig:dataset2g}
		}
		\caption{The setting of dataset collection. Each subject (a) appears at nine different locations on the images and (b) is instructed to gaze at 72 crosses and the camera at each location.}
		\label{fig:dataset2}
	\end{figure}
	\begin{figure}
		\centering
		\includegraphics[width=8.5cm]{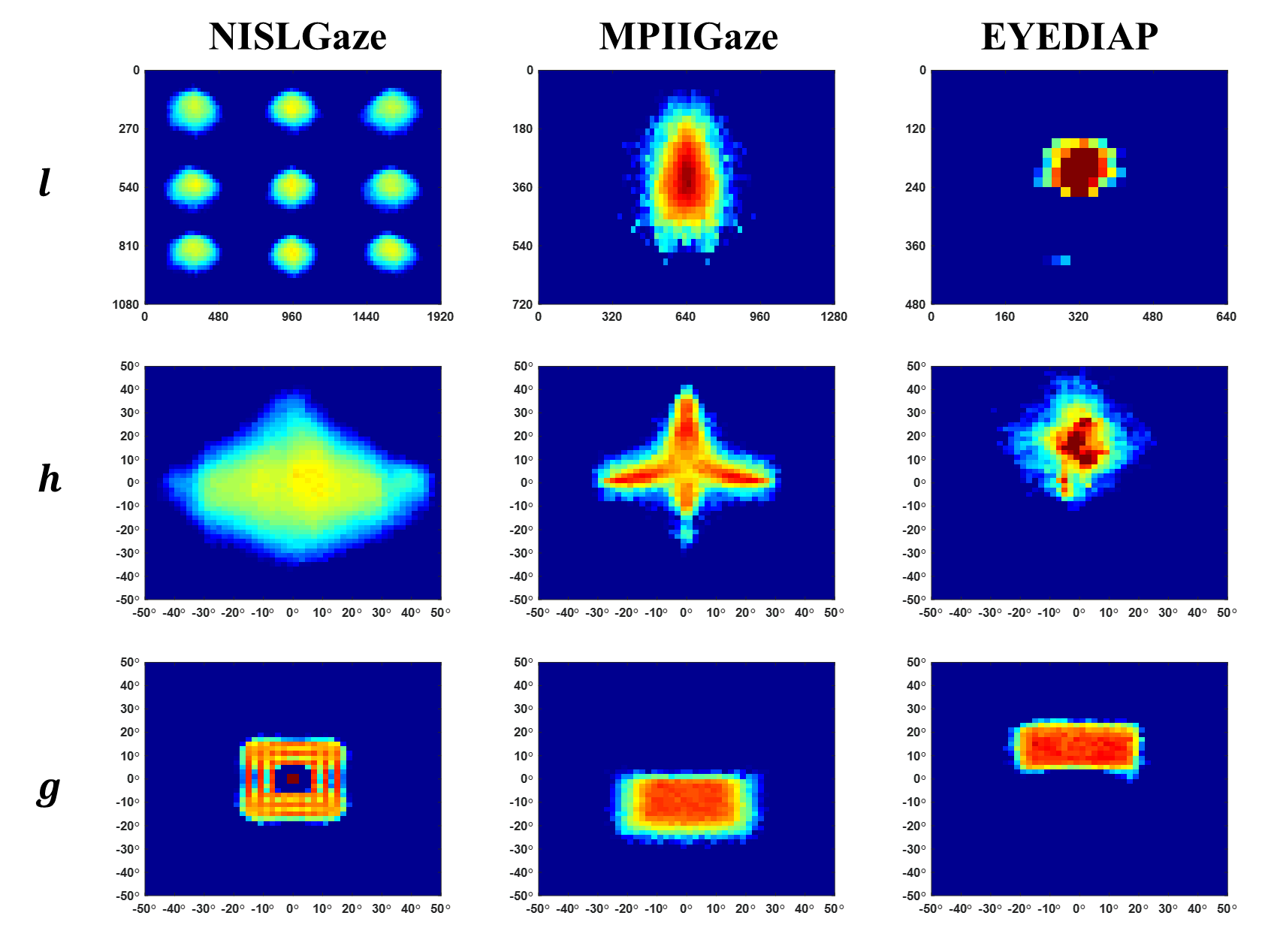}
		\caption{Statistics of the NISLGaze (ours), MPIIGaze and EYEDIAP datasets in terms of face location ($\bm{l}$), head pose ($\bm{h}$) and gaze directions ($\bm{g}$).}
		\label{fig:dataset1}
	\end{figure}

	\subsubsection{Post Processing}
	\label{sec: nisl}

	We extracted an \textit{Evaluation Subset} of images from the raw videos, which we used for training and testing. We first used OpenFace~\cite{baltrusaitis2018openface} to extract the facial landmarks and then sampled images from the video at 10 fps. We removed images for which either the landmark detector failed or both eyes were not detected. We also removed images containing blinks by setting a threshold on the eye aspect ratio calculated from the detected facial landmarks~\cite{cech2016real}. In addition, we removed images that were collected within 0.2 s after each notification sound. The resulting \textit{Evaluation Subset} contains $496,695$ images in total (about 2,500 images for each subject at each face location).

	We further defined 10 subsets of the \textit{Evaluation Set} according to the face location. Nine of the subsets contains the images of all subjects at one face location. We denote these by \textit{\LS~$(\alpha)$}, $\alpha\in\{\text{ul}, \text{uc}, \text{ur}, \text{cl}, \text{c}, \text{cr}, \text{ll}, \text{lc}, \text{lr}\}$ $\alpha$, where u means upper, l means lower or left depending upon whether it is in the first or second position, r means right and c means center. These nine subsets are disjoint. Their union is the \textit{Evaluation Set}. The tenth subset is the \textit{Mixed Location Subset}, which has a similar size as the \textit{Location Subsets}, but contains images of all subjects at all face locations. We obtained this dataset by downsampling the \textit{Evaluation Subset} by a factor of 9.

	We labeled our dataset by assuming that the centers of subjects' eyes were at a fixed location in space. To be specific, we created a right-handed coordinate system whose origin was located on the wall at the center of the gaze targets with the $z$-axis pointing into the wall and the $y$-axis pointed upwards. The locations of the crosses on the wall were in the form of ($x$, $y$, 0). The distance between two adjacent crosses was 6 cm. The camera was at (0, 0, -9.5 cm). We assumed that the subjects' eyes were fixed at (0,height of eye - 160 cm,-90 cm). Strictly speaking, this is not true, since subjects were rotating their heads while maintaining gaze on each target. However, deviations due to this rotation were small.

	\section{Experiments}

	We evaluated performance both with and without calibration. We evaluated on the MPIIGaze~\cite{zhang2015appearance}, the EYEDIAP~\cite{funes2014eyediap}, the ColumbiaGaze~\cite{smith2013gaze} and the NISLGaze datasets. We conducted both within- and cross-dataset evaluation.

	We evaluated two calibration conditions: multiple gaze target calibration (\textbf{MGTC}) and single gaze target calibration (\textbf{SGTC}). For MGTC, We set the number of image per gaze target $S=1$, but allowed the number of gaze targets $T$ to vary. For SGTC, we set $T=1$, but allowed $S$ to vary. For each calibration condition, we created calibration sets by sampling from the test set. We report performance (e.g. estimation error) computed over all images in the test set, except that in the calibration set.

	\begin{table}[!t]
		\centering
		\renewcommand\arraystretch{1.3}
		\caption{Mean Angular Errors for Gaze Estimation without Calibration.}
        \label{table:calFree}
		\begin{tabular}{c|c|c}
			\cline{1-3}
			Name&MPIIGaze&EYEDIAP\\
			\hline
			iTracker~\cite{krafka2016eye}& $6.2^\circ$& $8.3^\circ$\\
			\hline
			iTracker (AlexNet)~\cite{zhang2019mpiigaze}& $5.6^\circ$& \\
			\hline
			Spatial weight CNN~\cite{zhang2017s}& $4.8^\circ$& $6.0^\circ$\\
			\hline
			CNN with shape cue~\cite{palmero2018recurrent}&$4.8^\circ$ &$5.9^\circ$\\
			\hline
			RT-GENE~\cite{fischer2018rt}& $4.8^\circ$& \\
			\hline
			\hline
			Baseline CNN& $5.3^\circ$& $6.2^\circ$\\
			\hline
			Baseline CNN + GD& $4.7^\circ$& $5.9^\circ$\\
			\hline
			Dilated CNN~\cite{chen2018appearance}& $4.8^\circ$& $5.8^\circ$\\
			\hline
			GEDDNet (Dilated CNN + GD)& $\bm{4.5^\circ}$& $\bm{5.4^\circ}$\\
			\hline
			ResNet & $5.4^\circ$& \\
			\hline
			ResNet + GD & $5.2^\circ$ &  \\
			\hline
			Dilated ResNet & $5.3^\circ$& \\
			\hline
			Dilated ResNet + GD& $4.9^\circ$& \\
			\cline{1-3}
		\end{tabular}
	\end{table}

	\subsection{Gaze Estimation without Calibration}

	We evaluated the performance of four estimators trained with and without gaze decomposition and tested them without calibration. The four estimators were the proposed Dilated CNN, the baseline CNN without dilated convolutions shown in Fig.~\ref{fig:Archi}(c), a ResNet architecture, and a ResNet architecture with dilated convolutions. We included ResNet as an example of a more modern network in comparison with GEDDNet, which followed an older VGG-like architecture.

	We conducted these experiments on the MPIIGaze and EYEDIAP datasets. For the MPIIGaze dataset, we conducted 15-fold leave-one-subject-out cross-validation. For the EYEDIAP dataset, we followed the protocol described in~\cite{zhang2017s}, namely, five-fold cross-validation on four VGA videos (both continuous and discrete screen targets with both static and dynamic head pose) sampled at 2 fps, resulting in about 1,200 images per subject.

	We also compared our method with prior work: iTracker~\cite{krafka2016eye,zhang2017s}, spatial weights CNN~\cite{zhang2017s}, CNN with shape cue~\cite{palmero2018recurrent}, RT-GENE without ensembling~\cite{fischer2018rt} and Dilated-Net~\cite{chen2018appearance}. All of these methods used face images (or face plus eye images) as input. We re-implemented the CNN with shape cue and used the originally published results for the other methods.

	The results are shown in TABLE~\ref{table:calFree}. On MPIIGaze, the GEDDnet (Dilated CNN + GD) achieved a $4.5^\circ$ mean angular error, outperforming the state-of-the-art $4.8^\circ$~\cite{chen2018appearance,fischer2018rt,palmero2018recurrent,zhang2017s} by $6.3\%$. On EYEDIAP, it achieved $5.4^\circ$, outperforming the state-of-the-art $5.9^\circ$~\cite{palmero2018recurrent} by $8.5\%$. The GEDDnet outperformed the deeper ResNet networks, possibly due to overfitting. Thus, we did not evaluate it on EYEDIAP.

	Surprisingly, estimators trained with gaze decomposition outperformed those trained without. On the MPIIGaze dataset, training with gaze decomposition reduced test error by 11\% (Baseline CNN), 6\% (Dilated CNN), 4\% (ResNet) and 7.5 \% (Dilated ResNet). Conventional wisdom would predict the opposite due to train-test mismatch: applying gaze-decomposition during training corrects gaze estimates with a bias offset that is not used at test time.

	These results also demonstrate the advantage of dilated convolutions. For both our proposed VGG-like architecture and the deeper ResNet architecture and both with and without training with gaze decomposition, replacing max-pooling by dilated convolutions reduced error.
	We verified that the advantage was maintained in a cross-dataset setting by training on the Mixed Location Subset of NISLGaze and testing on MPIIGaze. The error for GEDDnet was $7.2^\circ$, $0.4^\circ$ lower than the $7.6^\circ$ error for the baseline CNN. These cross-dataset errors are smaller than the $7.7^\circ$ reported in~\cite{fischer2018rt}, where they trained on the RT-GENE dataset and tested on MPIIGaze.

    \subsection{Multiple Gaze Target Calibration}

	\begin{table*}[t]
		\caption{Estimation Error (mean $\pm$ SD in degree) of Multiple Gaze Target Calibration.}
        \label{tab:MPIIMGTC}
		\begin{center}
			\renewcommand{\arraystretch}{1.3}
			\begin{tabular}{c|c|ccccccc}
				\cline{1-9}
				\multirow{2}*{Face+Eye}&\multirow{2}*{Backbone}&\multicolumn{7}{c}{Number of gaze targets $T$ ($S = 1$)}\\
				\cline{3-9}
				&&1&5&9&16&32&64&128\\
				\cline{1-9}
				FC&Dilated+GD&$4.8\pm0.8$&$5.5\pm1.5$&$4.5\pm1.0$&$3.5\pm0.5$&$2.9\pm0.2$&$2.7\pm0.1$&$2.5\pm0.1$\\
				\cline{1-9}
				LA~\cite{liu2018differential}&Dilated+GD&\textit{NA}&$4.7\pm3.0$&$3.1\pm0.7$&$2.8\pm0.3$&$\mathbf{2.6}\pm0.1$&$\mathbf{2.5} \pm 0.1$&$\mathbf{2.4}\pm0.0$\\
				\cline{1-9}
				PA~\cite{zhang2019evaluation}&Dilated+GD&\textit{NA}&\textit{NA}&$9.2\pm3.5$&$3.8\pm1.3$&$2.8\pm0.3$&$2.6\pm0.1$&$\mathbf{2.4}\pm0.0$\\
				\cline{1-9}
				LP~\cite{linden2018appearance}&Dilated&$4.2\pm1.3$&$3.3\pm0.4$&$3.2\pm0.3$&$3.0\pm0.1$&$3.0\pm0.1$&$2.9\pm0.1$&$2.9\pm0.0$\\
				\cline{1-9}
				DF~\cite{liu2019differential}&Dilated&$4.2\pm1.5$&$3.2\pm0.4$&$3.0\pm0.3$&$3.0\pm0.1$&$\mathbf{2.6}\pm0.1$&$2.6\pm0.0$&$2.6\pm0.0$\\
				\cline{1-9}
				Baseline CNN&Baseline CNN&$4.3\pm0.9$&$3.6\pm0.5$&$3.5\pm0.4$&$3.4\pm0.3$&$3.4\pm0.1$&$3.3\pm0.0$&$3.3\pm0.0$\\
				\cline{1-9}
				Baseline CNN+GD&Baseline+GD&$3.7\pm0.9$&$3.1\pm0.4$&$2.9\pm0.2$&$2.9\pm0.1$&$2.8\pm0.1$&$2.8\pm0.0$&$2.8\pm0.0$\\
				\cline{1-9}
				Dilated CNN~\cite{chen2018appearance}&Dilated&$3.7\pm0.9$&$3.1\pm0.4$&$3.0\pm0.2$&$2.9\pm0.2$&$2.9\pm0.1$&$2.9\pm0.0$&$2.9\pm0.0$\\
				\cline{1-9}
				Dilated CNN+GD (ML)~\cite{chen2020offset}&Dilated+GD&$3.7\pm1.4$&$2.9\pm0.4$&$\mathbf{2.7}\pm0.2$&$\mathbf{2.7}\pm0.1$&$\mathbf{2.6}\pm0.1$&$2.6\pm0.0$&$2.6\pm0.0$\\
				\cline{1-9}
				GEDDNet (Dilated CNN+GD)&Dilated+GD&$\mathbf{3.5}\pm0.9$&$\mathbf{2.8}\pm0.4$&$\mathbf{2.7}\pm0.2$&$\mathbf{2.7}\pm0.1$&$\mathbf{2.6}\pm0.1$&$2.6\pm0.0$&$2.6\pm0.0$\\
				\cline{1-9}
				\hline
				\hline
				\multicolumn{2}{c}{Eye only}&\multicolumn{2}{l}{}&&&&&\\
				\cline{1-9}
				GRS~\cite{yu2019improving}$^*$&VGG16&5.0&4.2&4.0&&&&\\
				\cline{1-9}
				FAZE~\cite{park2019few}$^*$&DenseNet&4.7&4.0&3.9&3.8&3.8&3.7&3.7\\
				\cline{1-9}
				Dilated CNN+GD$^*$&Dilated+GD&$\mathbf{4.1}\pm0.9$&$\mathbf{3.5}\pm0.4$&$\mathbf{3.4}\pm0.3$&$\mathbf{3.3}\pm0.3$&$\mathbf{3.3}\pm0.1$&$\mathbf{3.3}\pm0.0$&$\mathbf{3.3}\pm0.0$\\
				\cline{1-9}
				\multicolumn{9}{l}{\small{$*$: The three methods at the bottom only used eye images as input, while the others used face+eye images as input.}}
			\end{tabular}
		\end{center}
	\end{table*}

	We compared the estimation error of our approach with that of other existing calibration methods: fine-tuning the last FC layer (\textbf{FC}), linear adaptation (\textbf{LA})~\cite{liu2018differential}, third order polynomial adaptation (\textbf{PA})~\cite{zhang2019evaluation}, the differential method (\textbf{DF})~\cite{liu2019differential}, fine-tuning the latent parameters (\textbf{LP})~\cite{linden2018appearance}, gaze redirection synthesis (\textbf{GRS})~\cite{yu2019improving} and \textbf{FAZE}~\cite{park2019few}. FC, LA and PA were applied to our network trained with gaze decomposition. DF and LP were re-implemented in a network with the same architecture as in Fig.~\ref{fig:Archi}. For GRS and FAZE, we used their originally published results in the same within-dataset leave-one-subject-out cross-validation. Since both GRS and FAZE only considered eye images as input, we trained a network without the face component in Fig.~\ref{fig:Archi} for fair comparison.

	For these experiments we used the MPIIGaze dataset. $T$ images from the test set were randomly selected as $\mathcal{D}_m$ ($S=P=1$). We evaluated gaze estimation error on the remaining test set images. 	Estimation errors reported were averaged over all subjects and over 5,000 calibration iterations per subject

	The results are shown in Table~\ref{tab:MPIIMGTC}. For all methods, as the number of calibration gaze targets increases, the error decreases. Our proposed method performed the best for low complexity calibration sets, outperforming other methods when the number of gaze targets was less than or equal to 32.

	The MAP estimate of the bias described in Eq.~\eqref{cali} resulted in lower estimation error than the ML estimate we proposed previously in~\cite{chen2020offset}, especially when $|\D_m|$ was small. For example, with only one calibration image, the error when using the proposed MAP estimate was $3.5$, or $0.2^\circ$ ($5.4\%$) lower than that with the ML estimate.

    When testing with calibration, training with gaze decomposition reduces error by $0.5^\circ$ to $0.6^\circ$ for the Baseline CNN and by $0.2^\circ$ to $0.3^\circ$ for the dilated CNN, in comparison to training without gaze decomposition. These results are expected, since train and test conditions are matched.
	Similar to our results without calibration, replacing max-pooling by dilated convolutions by dilated convolutions improved performance. Without gaze decomposition, the Dilated CNN outperformed the Baseline CNN by $0.4^\circ$ to $0.6^\circ$.  With gaze decomposition, the Dilated CNN outperformed the Baseline CNN by $0.2^\circ$ to $0.33^\circ$

    \subsection{Single Gaze Target Calibration}

    This section more extensively characterizes the performance of GEDDNet for Single Gaze Target Calibration. Since this is the simplest and least time consuming way of collecting calibration data, it is likely to be the one most preferred by users.

    \subsubsection{Size of calibration set}

	\begin{table*}[!t]
		\caption{Estimation Error (mean $\pm$ SD in degrees) of SGTC on MPIIGaze.}
		\begin{center}
			\renewcommand{\arraystretch}{1.3}
			\begin{tabular}{c|c|cccc}
				\cline{1-6}
				\multirow{2}*{}&\multirow{2}*{Backbone}&\multicolumn{4}{c}{Number of images per gaze target $S$}\\
				\cline{3-6}
				&&1&5&9&16\\
				\cline{1-6}
				FC&Dilated+GD&$4.8\pm0.8$&$7.1\pm1.2$&$8.0\pm1.2$&$8.4\pm1.1$\\
				\cline{1-6}
				LA~\cite{liu2018differential}&Dilated+GD&\textit{NA}&$15.1\pm2.7$&$14.9\pm2.5$&$14.3\pm2.1$\\
				\cline{1-6}
				PA~\cite{zhang2019evaluation}&Dilated+GD&\textit{NA}&\textit{NA}&$14.7\pm2.4$&$14.3\pm2.1$\\
				\cline{1-6}
				LP~\cite{linden2018appearance}&Dilated&$4.2\pm1.3$&$3.9\pm0.8$&$3.8\pm0.7$&$3.7\pm0.6$\\
				\cline{1-6}
				DF~\cite{liu2019differential}&Dilated&$4.2\pm1.5$&$3.5\pm0.7$&$3.4\pm0.5$&$3.3\pm0.4$\\
				\cline{1-6}
				Baseline CNN&Baseline CNN&$4.3\pm0.9$&$4.0\pm0.7$&$4.0\pm0.6$&$3.9\pm0.5$\\
				\cline{1-6}
				Baseline CNN+GD&Baseline CNN&$3.7\pm0.9$&$3.3\pm0.7$&$3.2\pm0.5$&$3.1\pm0.4$\\
				\cline{1-6}
				Dilated CNN~\cite{chen2018appearance}&Dilated&$3.7\pm0.9$&$3.4\pm0.6$&$3.3\pm0.5$&$3.2\pm0.4$\\
				\cline{1-6}
				Dilated CNN+GD (ML)~\cite{chen2020offset}&Dilated+GD&$3.7\pm1.4$&$3.1\pm0.6$&$\mathbf{3.0}\pm0.4$&$\mathbf{2.9}\pm0.3$\\
				\cline{1-6}
				GEDDNet (Dilated CNN+GD)&Dilated+GD&$\mathbf{3.5}\pm0.9$&$\mathbf{3.0}\pm0.5$&$\mathbf{3.0}\pm0.4$&$\mathbf{2.9}\pm0.3$\\
				\cline{1-6}
			\end{tabular}
		\end{center}
		\label{tab:MPIISGTC}
	\end{table*}

    We first evaluated the effect of the number of images used for calibration using the MPIIGaze dataset ($T=P=1$). We first randomly selected a calibration target in the 2D (yaw, pitch) gaze space. We then randomly selected $S$ images whose true gaze angles differed from the calibration target by less than $2^\circ$ as $\mathcal{D}_m$. We discarded the gaze target if fewer than $S$ images met the $2^\circ$ requirement. Our experimental results on the EYEDIAP dataset were consistent with those on the MPIIGaze dataset. We include the results and detailed analysis in Supplementary Material.

	Results are shown in Table~\ref{tab:MPIISGTC}. The GEDDnet (Dilated CNN + GD) resulted in the best performance no matter how many images were used for training. For SGTC, we do not expect the performance of more complex algorithms to eventually outperform the GEDDnet. The key advantage of those networks is that the additional parameters allow for adjustment of gaze estimates in different directions. However, if there is only a single target, there is not enough information available in the calibration dataset to exploit this additional power.

	Consistent with our previous results, replacing max-pooling with dilated convolutions improved performance. Training with gaze decomposition was better than training without.

    We used experiments with the NISLGaze dataset to predict GEDDnet's performance under conditions more similar to those in practice. Experiments with MPIIGaze may not be fully reflective of expected performance in the field, since the images in the calibration set were taken at different times and for slightly different targets. In the field, we would expect that the calibration data would be collected during a single calibration period before the start of the eye tracker use.

    To examine the performance of SGTC under these conditions, we used calibration data taken from a single video as the subject gazes at the camera with their head in the center of the image. We conducted five-fold inter-subject cross-validation, where we split the folds by subject. Since there were 21 subjects, four folds contained data from four subjects and the fifth contained data from five. All training was performed with gaze decomposition, where we trained on a ``Multiple Location'' dataset formed by the union of the MPIIGaze, the EYEDIAP and the NISLGaze (\textit{Mixed Location Subset}) datasets. Testing was performed on all data from the \textit{Evaluation Subset} of NISLGaze, except the video used for calibration.

    In the NISLGaze dataset, each video contains a images of a subject with a variety of head orientations, but the same head position. To vary head orientation diversity in the calibration set, we changed both the number of samples extracted from the video and the sample period (the time between successive samples).

    \begin{table*}
		\caption{Estimation error (mean$\pm$SD) for SGTC with varying calibration data amounts. \textbf{Random} means randomly sampling images from the calibration video; \textbf{All} means all samples from the video. \textbf{NA}: videos not long enough to collect 32 samples space by 0.5 seconds.}
		\label{table:NISLDuration}
        \begin{center}
    		\begin{tabular}{c|cccccc}
    		    \cline{1-7}
    			\diagbox{\makecell{sample\\number}}{\makecell{sample\\period}}&$0.1 s$&$0.2 s$&$0.5 s$&Random&All&\makecell{No\\Calibration}\\
    			\cline{1-7}
    			8&$3.38^\circ\pm0.10^\circ$&$3.27^\circ\pm 0.06^\circ$&$3.16^\circ\pm0.04^\circ$&$3.16^\circ \pm 0.03^\circ$&\multirow{3}*{$3.08^\circ$}&\multirow{3}*{$4.53^\circ$}\\
    			\cline{1-5}
    			16&$3.28^\circ\pm0.07^\circ$&$3.18^\circ\pm 0.04^\circ$&$3.12^\circ\pm0.02^\circ$&$3.12^\circ \pm 0.02^\circ$&&\\
    			\cline{1-5}
    			32&$3.18^\circ\pm0.04^\circ$&$3.13^\circ\pm 0.02^\circ$&NA&$3.10^\circ \pm 0.01^\circ$&&\\
    			\cline{1-7}
    		\end{tabular}
		\end{center}
	\end{table*}

	The results are shown in Table~\ref{table:NISLDuration}. For the same number of samples, increasing the time between successive samples reduces the estimation error, due to larger variations in head pose. Using 16 samples with sampling period of 0.1s resulted in an average error of $3.28^\circ$, an improvement of $1.25^\circ$ compared to the error without calibration, $4.53^\circ$. This is $86.2\%$ of the improvement when using all samples from the video for calibration. We can reduce the estimation error significantly (by $38.1\%$) using only a short (1.6s) calibration period.

    \subsubsection{Robustness to calibration target location}
    \label{sec:target location robustness}

    To evaluate the robustness of SGTC to the location of calibration target, we evaluated the performance on MPIIGaze when the calibration target was located in 24 different $5^\circ \times 5^\circ$ regions. We set $S=9$, since the errors in Table~\ref{tab:MPIISGTC} saturated beyond this point. We also calculated a lower bound $\underline{E}$ by estimating the bias of each subject from all images in the test set and evaluating on the test set. We denote the mean angular error without calibration by $\overline{E}$.

    Fig.~\ref{fig:MPII2} shows the results. The error with the calibration target at the center of the gaze range was lower than the error with the target at the boundary. However, SGTC is quite robust to the location of calibration target. The standard deviation over regions was only $0.16^\circ$.

	\begin{figure}
		\centering
		\includegraphics[width=\columnwidth]{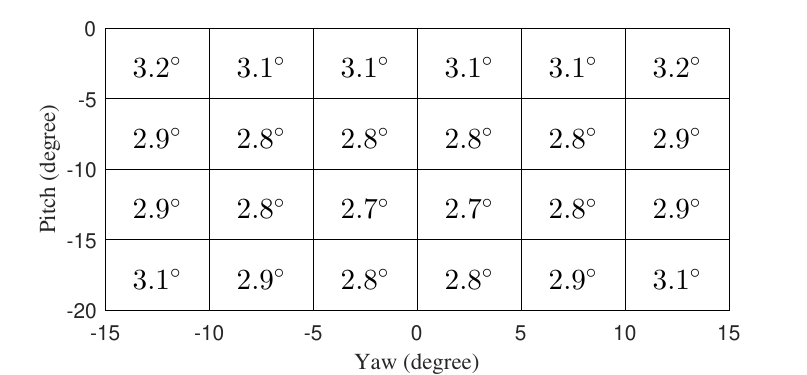}
		\caption{Within-dataset evaluation. The mean angular error on MPIIGaze for SGTC ($S=9$). The number in each box shows the mean angular error when the calibration gaze target is located inside the corresponding $5\times 5$ region. Mean errors are computed by averaging over the entire test set (except for the calibration set) and over 100 calibration gaze targets located at a $10\times10$ grid spaced by $0.5^\circ$ in yaw and pitch inside each region. $\overline{E}=4.5^\circ, \underline{E}=2.6^\circ$.}
		\label{fig:MPII2}
	\end{figure}

    \begin{figure}[t]
		\centering
		\includegraphics[width=8.5cm]{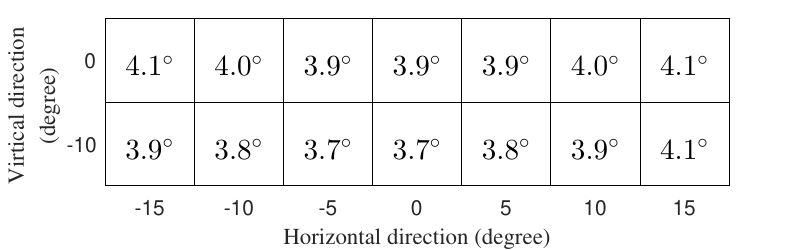}
		\caption{Cross-dataset evaluation. Mean angular error of SGTC ($S=5$) when calibrated at different gaze targets. Trained on MPIIGaze and NISLGaze. Tested on the ColumbiaGaze. $\overline{E}=4.8^\circ, \underline{E}=3.6^\circ$.}
		\label{fig:Colum}
	\end{figure}

	We confirmed that robustness to calibration target location was maintained in a cross-dataset setting. We trained GEDDNet on the combination of MPIIGaze and NISLGaze, and tested on the ColumbiaGaze dataset. We excluded the images corresponding to $10^\circ$ vertical angle during testing, since the MPIIGaze dataset mainly covers pitch angles from $-20^\circ$ to $0^\circ$.

	Fig.~\ref{fig:Colum} shows the error of SGTC when calibrating at different gaze targets. We used the five images gazing at the same gaze target but with different head poses as the calibration set ($S=5$, $T=1$). The error raged from $3.7^\circ$ to $4.1^\circ$, which was $14.6\%$ to $22.9\%$ lower than without calibration ($\overline{E}=4.8^\circ$).

	\begin{table*}
		\centering
		\caption{Estimation error (mean$\pm$SD) for SGTC at different head locations. \textbf{Overall}: samples from all nine locations. \textbf{Center}: samples from center location. \textbf{Periphery}: samples locations outside the center.}
		\label{table:NISLResLocation}
		\begin{center} \begin{tabular}{c|ccc|ccc}
		    \cline{1-7}
			&\multicolumn{3}{c|}{Multiple head location training}&\multicolumn{3}{c}{Center head location training}\\
			\cline{2-7}
			&\makecell{Overall\\Error}&\makecell{Center\\Error}&\makecell{Periphery\\Error}&\makecell{Overall\\Error}&\makecell{Center\\Error}&\makecell{Periphery\\Error}\\
			\cline{1-7}
			Without Calibration&$4.53^\circ\pm0.09^\circ$&$4.55^\circ$&$4.53^\circ\pm0.09^\circ$&$4.61^\circ\pm0.14^\circ$&$4.44^\circ$&$4.64^\circ\pm0.13^\circ$\\
			\cline{1-7}
			Center location calibration&$3.28^\circ\pm0.11^\circ$&$3.18^\circ$&$3.29^\circ\pm0.11^\circ$&$3.51^\circ\pm0.20^\circ$&$3.18^\circ$&$3.55^\circ\pm0.16^\circ$\\
			\cline{1-7}
			Matched location calibration&$3.27^\circ\pm0.11^\circ$&$3.18^\circ$&$3.28^\circ\pm0.11^\circ$&$3.52^\circ\pm0.21^\circ$&$3.18^\circ$&$3.57^\circ\pm0.18^\circ$\\
			\cline{1-7}
		\end{tabular} \end{center}
	\end{table*}

	\subsubsection{Robustness to head position}
	\label{sec:headposition}

	We ran four experiments to evaluate effect of head location following a two-by-two design, where one factor was calibration location and the other factor was training location.

	For calibration, we considered \textbf{center location calibration}, where calibration parameters were estimated only with the head in the center of the image, but used to correct estimates for images with the head at all locations, and \textbf{matched location calibration}, where separate calibration parameters were estimated for each of the nine head locations. Center location calibration is the most convenient scenario for practical deployment. Matched location calibration enables us to investigate the possibility that, similar to PCCR eye trackers, different calibration parameters would be better for different head locations.

	For training, we considered \textbf{multiple head location training}, where training data included images with the head at all nine image locations, and \textbf{center head location training}, where training data included only images with the head at the center. The dataset sizes in terms of NISLGaze for the two conditions was matched (55,275 and 55,858 images respectively), as the training set of multiple head location dataset downsampled the Evaluation Subset by a factor of 9. Center head location training is the scenario most common in current work, since most databases center the head in the image. Multiple head location training enables us to test whether normalization can fully compensate for off-center head images. We also include the MPIIGaze and the EYEDIAP datasets for training in both conditions, where most faces appear in the center of the image in these two datasets.

	Based on the results in Table~\ref{table:NISLDuration}, per-trial calibration data was obtained by taking 16 consecutive frames sampled at 10 fps while the subject gazed at the camera (1.6 seconds total), but moved their head. Calibration data was excluded from the testing data. Estimation error was averaged over 500 trials, over all subjects, and over all folds.

	Table~\ref{table:NISLResLocation} shows the estimation errors. Standard deviations were computed over face locations. We observe almost no difference in estimation error between center location calibration and matched location calibration under both training conditions.

	This suggests that calibration parameters do not vary much with head location. To confirm this, we calculated the bias of the network trained on mixed locations for each subject and for each face location, by averaging the gaze estimates from all images taken as the subject gazed at the camera from that location. We observed wide inter-subject bias variations (from $-2.7^\circ$ to $3.4^\circ$ in yaw and from $-7.5^\circ$ to $5.6^\circ$ in pitch), but very little intra-subject variation across location (SDs ranged from $0.1^\circ$ to $0.5^\circ$ for yaw and from $0.2^\circ$ to $0.7^\circ$ for pitch). A repeated measures analysis of variance (ANOVA) did not indicate a statistically significant effect of the face location on the value of biases (yaw $F(8, 160)=1.63,p=.12$; pitch $F(8, 160)=1.48,p=.17$).

	On the other hand, this robustness was not observed for the training data. The overall error increased when training only at the center head location no matter whether or what type of calibration was applied. Almost all of the the increase is due to increased error for images with the head in the periphery.

	\subsubsection{Comparison with other methods}

	To ensure that the performance gains we observed for gaze decomposition versus other calibration networks were not limited to the MPIIGaze dataset, we compared the performance of GEDDnet with LP and DF in both within-dataset and cross-dataset settings. We restricted attention to LP and DF, as they were the best two competing methods for SGTC on MPIIGaze (Table~\ref{tab:MPIISGTC}).

	Table~\ref{table:NISLComp} shows that the performance gains by GEDDNet were maintained. In the within dataset setting, we trained on the combination of MPIIGaze, EYEDIAP and the Mixed Location Subset of NISLGaze and tested on the remainder of the Evaluation Subset of NISLGaze using matched location calibration. GEDDNet outperformed DF by $0.4^\circ$ ($10.8\%$) and LP by $1.1^\circ$ ($25.0\%$). The cross-dataset setting was the same as in the previous subsection: training on MPIIGaze and NISLGaze, and testing on ColumbiaGaze. GEDDNet outperformed DF by $0.3^\circ$ ($7.1\%$) and LP by $1.2^\circ$ ($23.5\%$).

	\begin{table}
		\centering
		\renewcommand\arraystretch{1.3}
		\caption{Comparison of Calibration Methods on NISLGaze ($T=1, S=16$ / $S=5$).}
		\label{table:NISLComp}
		\begin{tabular}{ccc}
			\cline{1-3}
			Method & Within-Dataset Error & Cross-Dataset Error\\
			\hline
			LP~\cite{linden2018appearance}&$4.4^\circ$&$5.1^\circ$\\
			\cline{1-3}
			DF~\cite{liu2019differential}&$3.7^\circ$&$4.2^\circ$ \\
			\cline{1-3}
			Baseline CNN&$3.7^\circ$&$4.4^\circ$\\
			\cline{1-3}
			GEDDNet&$3.3^\circ$&$3.9^\circ$\\
			\cline{1-3}
		\end{tabular}
	\end{table}

	\subsection{ Feature Map Sensitivity}

	Our previous results consistently show that replacing max-pooling with dilated-convolutions leads to improved gaze estimation performance. We expect the dilated-convolutional layers to be more sensitive to small changes of input and to capture variations due to gaze shifts better.

	To show this directly, we performed a sensitivity analysis on the feature maps in the final convolutional layers of the right eye networks trained on NISLGaze. Sensitivity coefficients were computed on images from the ColumbiaGaze dataset, since the head pose is well controlled.

	Let $\mathbf{X}(p,t) \in \mathbb{R}^{C \times H \times W}$ denote the response of the last convolutional layer of a network in response to an input with gaze direction $(p,t)$, where $p$ and $t$ are the pan and tilt angles. The parameters $C$, $H$ and $W$ denote the number of channels, height and width of the map. Given a pair of images with gaze angles $(0,t)$ and $(\Delta p, t)$ we define the sensitivity coefficient
	\begin{equation}
        S_{\Delta p, t} = \frac{||\mathbf{X}(\Delta p,t)-\mathbf{X}(0^\circ,t)||_F}{||\mathbf{X}(0^\circ,t)||_F}
    \end{equation}
    where $||\cdot||_F$ represents the Frobenius norm. The higher the value of $S_{\Delta p, t}$, the larger the differences between the feature maps.

    \begin{figure}
		\centering
		\includegraphics[width=7cm]{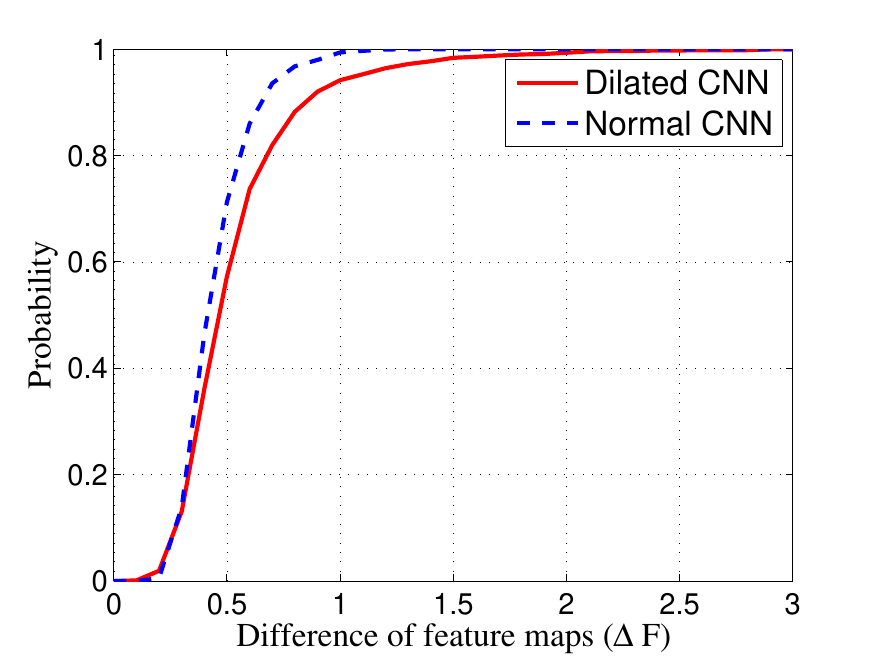}
        \caption{ Cumulative distribution function of the sensitivity coefficients computed over the last convolutional layer.}
		\label{fig:deltaF}
    \end{figure}

    The cumulative distribution functions of $S_{\Delta p,t}$ of all possible pair of images with $\Delta p \in \{-5^\circ, 5^\circ \}$ and $t \in \{ -10^\circ, 0^\circ, 10^\circ \}$ for the Baseline CNN and Dilated CNN networks
    are plotted in Fig.~\ref{fig:deltaF}. As expected, the last convolutional layer of the dilated CNN is more sensitive to image changes due to gaze shifts than the last convolutional layer of the Baseline CNN.

	\subsection{Influence of Landmark Detection}

	Although we performed data augmentation to mimic small errors in landmark estimation during training, we expect errors in gaze estimation to increase with errors in landmark estimation. We quantified this by within-dataset evaluation on the MPIIGaze dataset. We randomly perturbed the facial landmarks according to uniform distributions with maximum deviations of $3\%$, $5\%$, $7\%$ and $10\%$ of the eye size.

	The results are in Table~\ref{table:landmark}. As expected, gaze estimation error increases with landmark localization error. However, it is robust in the presence of small localization errors. When landmark localization error is $5\%$, the gaze estimation error is less than $5\%$. However, the gaze estimation error increases drastically for landmark localization errors greater than 7\%.

	\begin{table}
    	\centering
    	\caption{Mean Error under Different Landmark Disturbances.}
    	\begin{tabular}{ccccccc}
    		\cline{1-6}
    		Max. disturbance&$0\%$&$3\%$ & $5\%$ & $7\%$ & $10\%$\\
    		\cline{1-6}
    		Angular error & $4.5^\circ$&$4.6^\circ$&$4.7^\circ$&$5.2^\circ$&$6.0^\circ$\\
    		\cline{1-6}
    	\end{tabular}
    	\label{table:landmark}
    \end{table}

	\section{Theoretical Analysis}

	The improvement in test performance {\em without} calibration when the estimator was trained {\em with} gaze decomposition was surprising. Train and test conditions were mismatched: additive bias during training but no additive bias during testing. However, we observed this phenomenon consistently across all gaze estimation architectures we tried.

	To give a theoretical insight into why and under what conditions this might be expected, we analyzed a linear Gaussian model of image formation. We define $X_{i,j}\in\mathbb{R}^s$ to be the $j^\text{th}$ image of the $i^\text{th}$ subject, where $j$ ranges from 1 to $J$ (the total number of images per subject) and $i$ ranges from 1 to $I$ (the total number of subjects). We model the image formation process by
	\begin{align}
	X_{i,j} & =W_t t_{i,j}+W_c c_i+v_{i,j},
	\label{eq:gazeModel1} \\
	g_{i,j} & =t_{i,j}+b_i
	\label{eq:gazeModel2}
	\end{align}
	where $t_{i,j}\in\mathbb{R}^2$ is the direction of the optic axis, $c_i\in\mathbb{R}^p$ captures subject-dependent variations in appearance, $v_{i,j}\in\mathbb{R}^s$ is a noise term, $g_{i,j}\in\mathbb{R}^2$ is the direction of the visual axis (the gaze direction) and $b_i\in\mathbb{R}^2$ is a subject-dependent difference between the visual and optical axes. We assume that $s\geq 2+p$ and that both $W_t\in\mathbb{R}^{s\times 2}$ and $W_c\in\mathbb{R}^{s\times p}$ have full column rank. We assume that $t_{i,j}$, $c_i$, $v_{i,j}$, $b_i$ are zero-mean Gaussian random variables whose covariance matrices are denoted by $\Sigma_t$, $\Sigma_c$, $\Sigma_v$ and $\Sigma_b$, and that random variables with different values of $i$ and $j$ are independent.

	Assume a set of training data, $\Dtrain = \{ ( X_{i,j}, g_{i,j}) \mid i=1,2,\dots,\Itrain, j=1,2,\dots,\Jtrain\}$. For an estimator $\Khat$, we define the empirical mean squared error \emph{without} gaze decomposition by
	\begin{equation}
	\widehat{\MSE}_{\decbar} (\Khat) = \Tr \left\{
	\Ehat \left[ (g_{i,j}-\Khat X_{i,j}) (g_{i,j}-\Khat X_{i,j})^\top \right] \right\}
	\label{eq:MSEdecbar}
	\end{equation}
	where $\Ehat[\cdot]$ is the empirical expected value:
	\begin{equation}
	\Ehat_{}\big[A_{i,j}\big]=\frac{1}{IJ}\sum_{i=1}^{I}\sum_{j=1}^{J} A_{i,j},
	\label{eq:meanij}
	\end{equation}
	where $A_{i,j}$ is a symbolic variable. Eq.~\eqref{eq:MSEdecbar} is an empirical estimate of the mean squared error without calibration. We use the hat to indicate empirical estimates.

	For an estimator $\Khat$ and a set of subject dependent biases $\{ \hat{\mu}_i \}_{i=1}^{\Itrain}$, we define the empirical mean squared error \emph{with} gaze decomposition to be
	\begin{multline}
	\widehat{\MSE}_{\dec} (\Khat,\{ \hat{\mu}_i \}) =
	\Tr \Big\{ \Ehat \Big[ (g_{i,j}-(\Khat X_{i,j}+\hat{\mu}_i)) \times \\
	(g_{i,j}-(\Khat X_{i,j}+\hat{\mu}_i ))^\top \Big]\Big\},
	\label{eq:MSEdec}
	\end{multline}

	The optimal linear estimator that minimizes $\widehat{\MSE}_{\decbar}$ in Eq.~\eqref{eq:MSEdecbar} is referred to as the optimal empirical gaze estimator trained \emph{without} gaze decomposition. The optimal linear estimator that minimizes $\widehat{\MSE}_{\dec}$ in Eq.~\eqref{eq:MSEdec} is referred to as the optimal empirical gaze estimator trained \emph{with} gaze decomposition.

	In the Supplementary Material, we prove that

	\begin{enumerate}
	\item If variations in appearance due to gaze and due to subject are orthogonal ($W_t^\top W_c=\mathbf{0}$), then for any finite number of subjects $\Itrain$, the optimal empirical gaze estimator trained with gaze decomposition performs \emph{as well as or better than} the gaze estimator trained without gaze decomposition when the number of (gaze,image) pairs per subject $\Jtrain$ approaches infinity.

	\item If variations in appearance due to subject and gaze are not orthogonal ($W_t^\top W_c\neq\mathbf{0}$), then with enough data (when $\Itrain, \Jtrain\to\infty$) the optimal empirical gaze estimator trained with gaze decomposition performs worse than the gaze estimator trained without gaze decomposition.
	\end{enumerate}

	\noindent The second result may seem to negate the practical impact of the first, since it seems unlikely for gaze- and subject-dependent changes to be exactly orthogonal. However, this holds only in the limit as the amount of data becomes infinite. In practice, this will never be the case. We also show numerically in the Supplementary Material that if $\| W_t^\top W_c \|$ and the amount of training data ($\Itrain$) are small enough, then the optimal empirical gaze estimator trained with gaze decomposition can still outperform the gaze estimator trained without gaze decomposition.

	\section{Conclusion}

	This paper describes a systematic study of a low complexity algorithm for subject dependent calibration of appearance based gaze estimation. Our goal has been to investigate the factors affecting the performance of calibration, in order to identify the most important factors and to reduce the overhead associated with calibration.

	Based on this study, we advocate that a fairly simple algorithm, GEDDNet, which adds a subject dependent bias to the output of a subject independent gaze estimator, can be used the majority of applications. While more complex algorithms can result in better performance, the gains will not generally outweigh their cost. For example, Table~\ref{tab:MPIIMGTC} shows that more complex algorithms do not outperform this approach until the number of gaze targets is prohibitively large: greater than 32 targets. This is much larger than the nine points (commonly reduced to five for convenience) used for PCCR eye trackers. Even then, the gains are small. With 64 gaze targets, error of the best algorithm Linear Adaptation is only $0.1^\circ$ lower than our proposed GEDDNet ($2.5^\circ$ vs $2.6^\circ$).

	We argue that the use of only a single gaze target, the camera, for calibration is particularly compelling. This approach saves time, reduces cognitive load on the user, and imposes minimal hardware requirements. A camera is always required in appearance based approaches, and must always be visible to the user if it is to obtain reliable gaze estimates. This approach does not require any mechanism for directing the users gaze at towards multiple known locations in the environment. This both lowers cost and makes the system more widely applicable, e.g. to applications where gaze is directed into the environment, rather than onto a screen. Given the same number of images, the performance of GEDDNet with MGTC is only slightly better than with SGTC. Comparing Tables~\ref{tab:MPIIMGTC} and \ref{tab:MPIISGTC} for $S = T$, we find that SGTC achieved between $83\%$ to $100\%$ of the improvement over the performance without calibration (Table~\ref{table:calFree}) achieved by MGTC.

	Choosing the gaze target to be the camera is convenient and cost effective, but may seem a bit arbitrary. One may wonder whether or not there may be a better choice for a gaze target to improve the estimation of calibration parameters. Our results in Section~\ref{sec:target location robustness} indicate that GEDDNet is quite robust to the choice of the calibration target. Quite often, the camera will be centered in front of the user. Figures~\ref{fig:MPII2} and \ref{fig:Colum} show that calibration by looking at targets in the center of the field of view yield the best performance.

	Gaze decomposition performs better because of its simplicity. The small number of parameters used for calibration avoids overfitting. Despite its simplicity, it still compensates for the primary source of error, the subject-dependent bias discussed in Section~\ref{sec:target location robustness}. We suggested that this bias could be due to the offset between the visual axis and the optical axis. The former is the quantity of interest for gaze estimation, yet only the latter can be estimated from appearance. Since we did not measure this offset in our subjects, we cannot claim that the bias we observed experimentally is caused by that offset. However, our experimental results do support an internal, rather than external, origin for this bias. The biases estimated did not vary much with the head position during calibration as discussed in Section~\ref{sec:headposition}. Changes in head position did alter appearance, but did not change the bias.

	The aforementioned robustness to head position is another advantage of GEDDNet. The calibration parameters for PCCR eye trackers are very sensitive to head position, often necessitating re-calibration during long usage sessions. One hope for appearance based gaze estimators is that they will be less sensitive to head position, thus enabling fairly unconstrained gaze estimation. However, no work prior to ours investigated this hypothesis, in part due to the lack of eye gaze databases that had images of the same subject at a variety of different image positions. To address this gap, we collected the NISLGaze dataset, described here. Our experimental results with NISLGaze show GEDDNet with SGTC is not sensitive to head position during calibration. Comparing the bottom two rows in Table~\ref{table:NISLResLocation}, we see that using calibration parameters estimated with the head in the same position as during testing has no advantage over using calibration parameters extracted from images with the head in the center.

	However, our experimental results with NISLGaze reveal that changes in head position are not fully compensated by the normalization we are using~\cite{zhang2018revisiting}. Comparing the two columns of Table~\ref{table:NISLResLocation}, we see that a gaze estimator trained on faces collected from directly in front of the camera does not perform as well as a gaze estimator trained on faces originally at a variety of different position but normalized to appear as if they were taken from in front of the camera. The estimation error for faces in the periphery increases by 8\% higher when training only on center head locations. We speculate that the poorer performance for faces at the periphery may arise from nonlinear distortion and resulting differences in landmark detection. Nonlinear distortion is worse in the periphery, but is not compensated for in the normalization scheme. In a preliminary experiment, landmarks detected from the image before normalization and afterwards differ by $3\%$ on average. Although these effects are small, they may still cause significant errors because gaze shifts cause small changes in the eye images (Fig. \ref{fig:1}). This suggests several possibly promising avenue of future research, such as the development of improved normalization, adding position dependence to estimators, and adding positional variability to gaze datasets.

	Perhaps our most surprising experimental finding is that when using GEDDNet, the same estimator can be used both with and without calibration. This is certainly convenient, as it saves on weight storage costs and training, but it contradicts the prevailing wisdom that the best performance of an estimator is achieved when training and test conditions are matched. In particular, Table~\ref{table:calFree} shows that when used without calibration (setting $\hat{b}_i = 0$ in Eq.~\eqref{dec}), an estimator trained without gaze decomposition ($\hat{b}_i = 0$) performs \emph{worse} than the estimator trained with gaze decomposition. We are confident in this finding, as it was found irrespective of the deep network architecture used in the estimator (baseline CNN, dilated CNN or ResNet).

	We also provided theoretical justification for this finding. Using a linear Gaussian model of image formation, we proved that if subject-dependent and gaze-dependent changes in appearance lie in orthogonal subspaces, then the estimator with gaze decomposition always achieves lower MSE compared to the estimator without gaze decomposition. Although the actual image formation process is more complex and highly nonlinear, the linear analysis is analytically tractable and provides valuable theoretical insight into the potential reasons for this phenomenon.

	Intuitively, gaze- and subject-dependent changes may be orthogonal because gaze-dependent changes affect mostly the relative positions of the pupil in the eye, whereas subject-dependent changes will be spread over the entire face. While there may be subject-dependent changes in eye shape, these may be confounded with subject-independent changes in eye shape due to expression.

	If we assume that the observed bias is due to visual/optic axis offset, then another intuitive explanation is that the axis offset can be considered to be subject-dependent label noise. Ground-truth labels indicate the direction of the visual axis, but only the direction of the optic axis can be estimated from appearance. Thus, more reliable estimation performance would be obtained if we used the direction of the optic axis as ground truth labels. In this interpretation, gaze decomposition is a way to remove this label noise by exploiting prior knowledge that it is constant across all images of the same subject.

	Although we have focused our attention on the use of dilated convolutions to improve performance in the estimator we have used in our experiments, gaze decomposition can be applied to any network architecture for gaze estimation. For example, Tables~\ref{table:calFree}, \ref{tab:MPIIMGTC} and \ref{tab:MPIISGTC} all show that adding gaze decomposition to networks that do not use dilated convolutions improves performance both with and without calibration during deployment. However, our results do suggest that dilated convolutions are more sensitive to image changes due to gaze shifts (Figure~\ref{fig:deltaF}) and that this greater sensitivity leads to better performance in gaze estimation for all backbone networks we have tried.


	\ifCLASSOPTIONcaptionsoff
	\newpage
	\fi

	\bibliographystyle{IEEEtran}
	\bibliography{myreference}

	\begin{IEEEbiography}[{\includegraphics[width=1in,height=1.25in,clip,keepaspectratio]{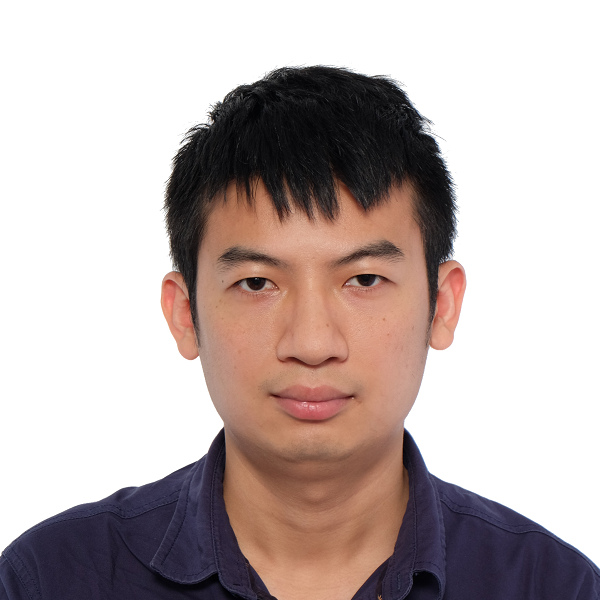}}]{Zhaokang Chen}
	received the B.S. degree from Fudan University, China, in 2014, and the Ph.D. degree from the Department of Electronic and Computer Engineering at the Hong Kong University of Science and Technology in 2020. His research interests include human computer interaction, eye tracking, computer vision and machine learning.
	\end{IEEEbiography}

	\begin{IEEEbiography}[{\includegraphics[width=1in,height=1.25in,clip,keepaspectratio]{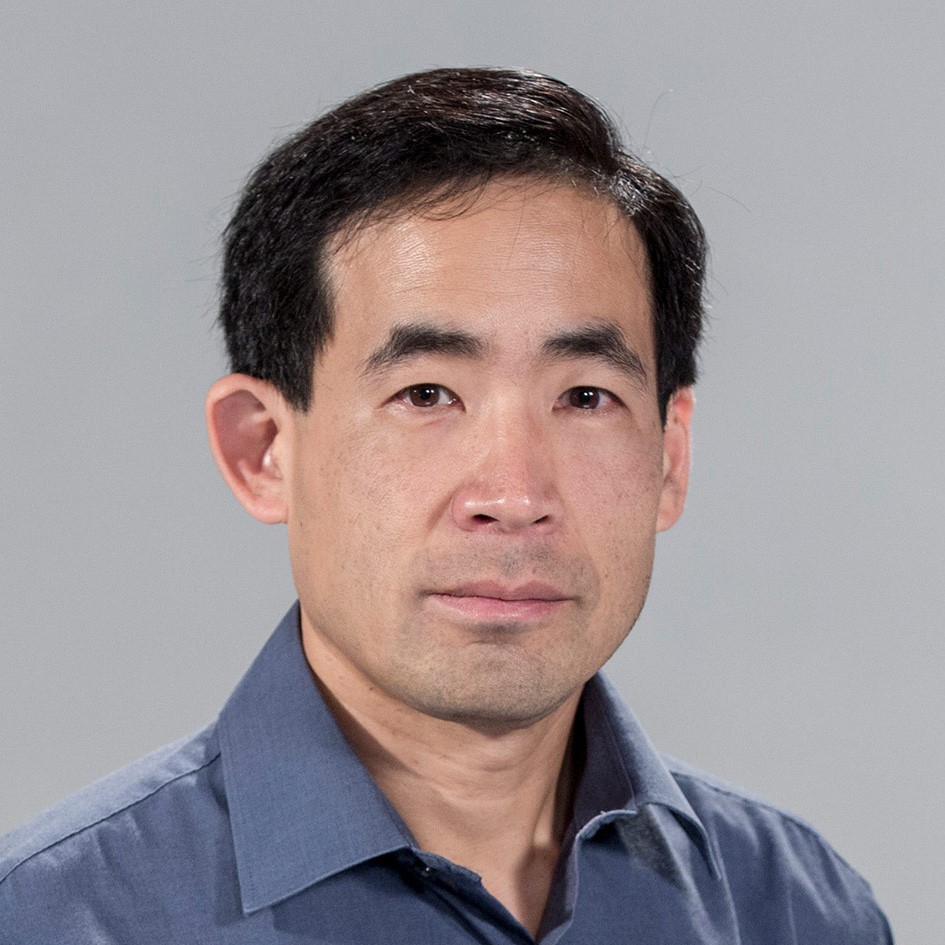}}]{Bertram E. Shi}

	(S'93-M'95-SM'00-F'01) is a Professor at the Department of Electronic and Computer Engineering at the Hong Kong University of Science and Technology. His research interests are in bio-inspired signal processing and robotics, neuromorphic engineering, computational neuroscience, developmental robotics, machine vision, image processing, and machine learning.  Prof. Shi twice served as Distinguished Lecturer for the IEEE Circuits and Systems Society.  He has also served as on the editorial boards of the IEEE Transactions on Circuits and Systems, the IEEE Transactions on Biomedical Circuits and Systems and Frontiers in Neuromorphic Engineering.  He served as Chair of the IEEE Circuits and Systems Society Technical Committee on Cellular Neural Networks and Array Computing and as General Chair and Technical Program Chair of conferences in that area.
	\end{IEEEbiography}

\end{document}


\title{Supplementary Material of ``Towards High Performance Low Complexity Calibration in Appearance Based Gaze Estimation''}

	\author{Zhaokang Chen and
		Bertram E. Shi,~\IEEEmembership{Fellow,~IEEE}
    }

	\markboth{}%
	{Z. Chen and B.E. Shi: Dilation and Decomposition for Gaze Estimation}


	\maketitle

	\IEEEpeerreviewmaketitle

	\IEEEraisesectionheading{\section{Detailed Theoretical Analysis}\label{sec:theory}}
	\subsection{Theoretical Analysis}
	In our model, we define $X_{i,j}\in\mathbb{R}^s$ to be the $j^\text{th}$ image of the $i^\text{th}$ subject, where $j$ ranges from 1 to $J$ (the total number of images per subject) and $i$ ranges from 1 to $I$ (the total number of subjects). We model the image formation process by
	\begin{align}
	X_{i,j} & =W_t t_{i,j}+W_c c_i+v_{i,j},
	\label{eq:gazeModel1} \\
	g_{i,j} & =t_{i,j}+b_i
	\label{eq:gazeModel2}
	\end{align}
	where $t_{i,j}\in\mathbb{R}^2$ is the direction of the optic axis, $c_i\in\mathbb{R}^p$ captures subject-dependent variations in appearance, $v_{i,j}\in\mathbb{R}^s$ is a noise term, $g_{i,j}\in\mathbb{R}^2$ is the direction of the visual axis (the gaze direction) and $b_i\in\mathbb{R}^2$ is a subject-dependent difference between the visual and optical axes. We assume that $s\geq 2+p$ and that both $W_t\in\mathbb{R}^{s\times 2}$ and $W_c\in\mathbb{R}^{s\times p}$ have full column rank. We assume that $t_{i,j}$, $c_i$, $v_{i,j}$, $b_i$ are zero-mean Gaussian random variables whose covariance matrices are denoted by $\Sigma_t$, $\Sigma_c$, $\Sigma_v$ and $\Sigma_b$, and that random variables with different values of $i$ and $j$ are independent.

	Let
	\begin{equation}
	\hat{g}_{i,j} = \mathcal{K} X_{i,j}
	\end{equation}
	be a linear subject-independent estimator of the gaze direction $g_{i,j}$ from an image $X_{i,j}$ where $\mathcal{K} \in \mathbb{R}^{2 \times s}$ and the hat indicates the quantity is an estimate. For convenience, we will refer to a subject-independent estimator using its $\mathcal{K}$ matrix.

	We define the mean squared error without calibration to be
	\begin{equation}
	\MSE_{\cbar} (\K) = \Tr \left\{
	\E \left[ (g_{i,j}-\K X_{i,j}) (g_{i,j}-\K X_{i,j})^\top \right] \right\},
	\label{eq:MSE}
	\end{equation}
	where $\E[\cdot]$ is the expected value, $\Tr$ is the matrix trace operator and $\top$  indicates transpose.
	Since the model is linear and all random variables are Gaussian and zero mean,
	the optimal estimator without calibration is \cite{kailath2000linear}
	\begin{equation}
	\K_{\cbar} = \Sigma_{gX} \Sigma_X^{-1}
	\label{eq:Kcbar}
	\end{equation}
	where
	\begin{align}
	\Sigma_{gX} & = \E [ g_{i,j} X_{i,j}^\top ] \\
	& = \Sigma_t W_t^\top, \\
	\Sigma_{X} & = \E [ X_{i,j} X_{i,j}^\top ] \\
	& = W_t \Sigma_t W_t^\top+W_c\Sigma_cW_c^\top+\Sigma_v.
	\label{eq:sigmax}
	\end{align}

	In an ideal case where the subject-dependent terms, $c_i$ and $b_i$ are known, we can calibrate the estimator by adding a subject dependent bias
	\begin{equation}
	\hat{g}_{i,j} = \mathcal{K} X_{i,j} + \mu_i,
	\end{equation}
	where
	\begin{equation}
	\mu_i = b_i-\K W_c c_i.
	\label{eq:mu}
	\end{equation}
	We define the mean squared error with ideal calibration to be
	\begin{multline}
	\MSE_{\calib} (\K) =
	\Tr \Big\{ \E \Big[ (g_{i,j}-(\K X_{i,j}+\mu_i)) \times \\
	(g_{i,j}-(\K X_{i,j}+\mu_i ))^\top \Big] \Big\} .
	\label{eq:MSEC}
	\end{multline}
	The optimal estimator with ideal calibration is
	\begin{equation}
	\K_{\calib} = C_{gX} C_{X}^{-1},
	\label{eq:Kcalib}
	\end{equation}
	where
	\begin{align}
	C_{gX} & =\E[ (g_{i,j}-b_i) (X_{i,j} - W_c c_i)^\top ] \\
	& = \Sigma_t W_t^\top, \\
	C_{X} & =  \E [ (X_{i,j} - W_c c_i) (X_{i,j} - W_c c_i)^\top ] \\
	& = W_t \Sigma_t W_t^\top + \Sigma_v
	\end{align}
	are covariance matrices where the data is centered by subject dependent offsets.
	The difference between the case without calibration and the case with calibration is that without calibration both $c_i$ and and $v_{i,j}$ are considered noise terms added to the image, but with ideal calibration, only $v_{i,j}$ is considered a noise term.

	In practice, the constants $W_t$, $W_c$, $\Sigma_t$, $\Sigma_c$ and $\Sigma_v$ are not known. Estimates of the matrices $\K_{\cbar}$ and $\K_{\calib}$ must be estimated from training data. Nor are the random variables $b_i$ and $c_i$ known. Rather, an estimate of the bias must be obtained from user calibration data.

	Assume a set of training data, $\Dtrain = \{ ( X_{i,j}, g_{i,j}) \mid i=1,2,\dots,\Itrain, j=1,2,\dots,\Jtrain\}$. For an estimator $\Khat$, we define the empirical mean squared error \emph{without} gaze decomposition by
	\begin{equation}
	\widehat{\MSE}_{\decbar} (\Khat) = \Tr \left\{
	\Ehat \left[ (g_{i,j}-\Khat X_{i,j}) (g_{i,j}-\Khat X_{i,j})^\top \right] \right\}
	\label{eq:MSEdecbar}
	\end{equation}
	where $\Ehat[\cdot]$ is the empirical expected value:
	\begin{equation}
	\Ehat_{}\big[A_{i,j}\big]=\frac{1}{IJ}\sum_{i=1}^{I}\sum_{j=1}^{J} A_{i,j},
	\label{eq:meanij}
	\end{equation}
	where $A_{i,j}$ is a symbolic variable. Eq.~\eqref{eq:MSEdecbar} is an empirical estimate of the mean squared error without calibration. We use the hat to indicate empirical estimates.
	The optimal estimator that minimizes $\widehat{\MSE}_{\decbar}$ is given by
	\begin{equation}
	\hat{\K}_{\decbar}= \hat{\Sigma}_{gX} \hat{\Sigma}_{X}^{-1},
	\label{eq:Kdecbar}
	\end{equation}
	where
	\begin{align}
	\hat{\Sigma}_{gX} & = \Ehat [ g_{i,j} X_{i,j}^\top ], \label{eq:Kdecbar1}\\
	\hat{\Sigma}_X & = \Ehat [ X_{i,j} X_{i,j}^\top ].
	\label{eq:Kdecbar2}
	\end{align}
	By the law of large numbers, $\hat{\K}_{\decbar} \to \K_{\cbar}$ when $\Itrain,\Jtrain\to\infty$.

	For an estimator $\Khat$ and a set of subject dependent biases $\{ \hat{\mu}_i \}_{i=1}^{\Itrain}$, we define the empirical mean squared error \emph{with} gaze decomposition
	\begin{multline}
	\widehat{\MSE}_{\dec} (\Khat,\{ \hat{\mu}_i \}) =
	\Tr \Big\{ \Ehat \Big[ (g_{i,j}-(\Khat X_{i,j}+\hat{\mu}_i)) \times \\
	(g_{i,j}-(\Khat X_{i,j}+\hat{\mu}_i ))^\top \Big]\Big\}.
	\label{eq:MSEdec}
	\end{multline}
	For the optimal set of biases, this is an empirical estimate of the mean squared error with ideal calibration.
	The optimal empirical estimator and set of biases that minimize $\widehat{\MSE}_{\dec}$ are given by
	\begin{align}
	\Khat_{\dec} = & \hat{C}_{gX} \hat{C}_{X}^{-1}, \label{eq:Kdec} \\
	\hat{\mu}_i = & \Ehat_j [g_{i,j}] - \Khat_{dec} \Ehat_j [X_{i,j}], \label{eq:muhat}
	\end{align}
	where
	\begin{align}
	\hat{C}_{gX} & = \Ehat[ (g_{i,j}-\Ehat_j [g_{i,j}])
	(X_{i,j} - \Ehat_j [X_{i,j}])^\top ] \\
	& = \Ehat[ g_{i,j} X_{i,j}^\top]
	- \Ehat_i \big[ \Ehat_j [g_{i,j}] \Ehat_j [X_{i,j}^\top]\big], \label{eq:Kdec2}\\
	\hat{C}_{X} & =  \Ehat [ (X_{i,j} - \Ehat_j [X_{i,j}])
	(X_{i,j} - \Ehat_j [X_{i,j}])^\top ] \\
	& = \Ehat [X_{i,j} X_{i,j}^\top ] - \Ehat_i \big[ \Ehat_j [X_{i,j}] \Ehat_j [X_{i,j}^\top]\big], \label{eq:Kdec3}
	\end{align}
	and we define the empirical mean operators
	\begin{equation}
	\Ehat_{i}\big[A_{i,j}\big] = \frac{1}{I}\sum_{i=1}^{I}A_{i,j} \textrm{ and }
	\Ehat_{j}\big[A_{i,j}\big] = \frac{1}{J}\sum_{j=1}^{J}A_{i,j}.
	\label{eq:meanj}
	\end{equation}
	The terms $\Ehat_j [g_{i,j}]$ and $\Ehat_j [X_{i,j}]$ can be considered to be empirical estimates of $b_i$ and $\hat{W_c c_i}$, respectively.
	By the law of large numbers, $\hatKdec\to\K_{\calib}$
	and $\hat{\mu}_i \to \mu_i$ as $\Itrain,\Jtrain\to\infty$.

	By optimality, it must be true that
	\begin{equation}
	\MSE_{\cbar} \big(\K_{\cbar}\big) \leq \MSE_{\cbar} \big( \K_{\calib}\big)
	\label{eq:inequality0}
	\end{equation}
	and
	\begin{equation}
	\MSEC\big(\K_{\calib}\big) \leq \MSEC\big( \K_{\cbar} \big).
	\end{equation}
	Thus, by the law of large numbers,
	\begin{equation}
	\lim_{\Itrain,\Jtrain\to\infty} \MSE_{\cbar} \big(\Khat_{\decbar}\big) \leq
	\lim_{\Itrain,\Jtrain\to\infty} \MSE_{\cbar} \big( \Khat_{\dec}\big)
	\label{eq:inequality1}
	\end{equation}
	and
	\begin{equation}
	\lim_{\Itrain,\Jtrain\to\infty} \MSEC\big( \Khat_{\dec} \big) \leq
	\lim_{\Itrain,\Jtrain\to\infty} \MSEC\big( \Khat_{\decbar} \big).
	\label{eq:inequality1_2}
	\end{equation}
	Eq.~\eqref{eq:inequality1_2} is consistent with our experimental results, where we observe an improvement in test performance with calibration when the network is trained with gaze decomposition. However, at first glance, it appears that Eq.~\eqref{eq:inequality1} is not, since we observe better test performance without calibration when the network is trained with gaze decomposition.

	The remainder of this section will discuss how to resolve this apparent contradiction. Briefly, we prove that if variations in appearance due to gaze and due to subject are orthogonal ($W_t^\top W_c=\mathbf{0}$), then the optimal estimators without calibration and with ideal calibration are identical, i.e. $\K_{\cbar}=\K_{\calib}$ (Theorem~\ref{theoOr1}). Thus,
	\begin{equation}
	\MSE_{\cbar} \big(\K_{\cbar}\big) = \MSE_{\cbar} \big( \K_{\calib}\big)
	\end{equation}
	and
	\begin{equation}
	\lim_{\Itrain,\Jtrain\to\infty} \MSE_{\cbar} \big(\Khat_{\decbar}\big) =
	\lim_{\Itrain,\Jtrain\to\infty} \MSE_{\cbar} \big( \Khat_{\dec} \big).
	\end{equation}
	Moreover, we prove that for any finite number of subjects $\Itrain$ and as the number of (gaze,image) pairs per subject $\Jtrain$ approaches infinity, the gaze estimator trained with gaze decomposition performs \emph{as well as or better} than the gaze estimator trained without gaze decomposition. In other words, the direction of the inequality in Eq.~\eqref{eq:inequality1} is \emph{reversed} (Theorem~\ref{theoOr2}):
	\begin{equation}
	\lim_{\Jtrain\to\infty} \MSE_{\cbar} \big(\Khat_{\decbar}\big) \geq
	\lim_{\Jtrain\to\infty} \MSE_{\cbar} \big( \Khat_{\dec} \big).
	\label{eq:inequality2}
	\end{equation}

	In the case where the variations in appearance and gaze are not orthogonal ($W_t^\top W_c\neq\mathbf{0}$), then the inequalities in Eqs.~\eqref{eq:inequality0} and \eqref{eq:inequality1} are strict. Nonetheless, our numerical results with this model still show that if the sizes of $W_t^\top W_c$ and $\Itrain$ are small enough, then the inequality Eq.~\eqref{eq:inequality2} still holds.
	These results resolve the apparent contradiction with our experimental results raised above, by suggesting that in practice the variations in appearance due to gaze and due to subject are close to orthogonal.

    The following theorem shows that under certain conditions $\K_{\cbar}$ and $\K_{\calib}$ are identical.
	\begin{theorem}
		If $W_t^\top W_c=\bm{0}$ and the elements of the noise vector $v_{i,j}$ are independent and identically distributed, i.e., $\Sigma_v=\sigma_v^2 \rm{I}_k$, then $\K_{\cbar}=\K_{\calib}$.
		\label{theoOr1}
	\end{theorem}
	\begin{proof}
		Combining Eq. \eqref{eq:sigmax} with $\Sigma_v=\sigma_v^2 \rm{I}_k$ and the definitions $U=[W_t\quad W_c]$ and
		\begin{equation}
		    \Sigma=\begin{bmatrix}\sigma_v^{-2}\Sigma_t &\bm{0}\\\bm{0} & \sigma_v^{-2}\Sigma_c\end{bmatrix},
		\end{equation}
		we can write $\Sigma_X^{-1}=\sigma_v^{-2}(\rm{I}+U\Sigma U^\top)^{-1}$.
		Using the Woodbury matrix identity~\cite{woodbury1950inverting},
		\begin{equation}
		    \Sigma_X^{-1} = \sigma_v^{-2}(\rm{I}-U(\Sigma^{-1}+U^\top U)^{-1} U^\top).
		    \label{eq:sigmaxinv}
		\end{equation}
		Since $W_t^\top W_c=\bm{0}$,
		\begin{align}
		    (\Sigma^{-1} & + U^\top U)^{-1} \nonumber \\
		    & = \left(
		    \begin{bmatrix}
		        \sigma_v^{2}\Sigma_t^{-1} &\bm{0} \\
		        \bm{0} & \sigma_v^{2}\Sigma_c^{-1}
		    \end{bmatrix}
		    +
		    \begin{bmatrix}
		        W_t^\top \\
		        W_c^\top
		    \end{bmatrix}
		    \begin{bmatrix} W_t & W_c \end{bmatrix} \right)^{-1} \\
		    & =
            \begin{bmatrix}
		        \sigma_v^{2}\Sigma_t^{-1}+W_t^\top W_t & \bm{0} \\
		        \bm{0} & \sigma_v^{2}\Sigma_c^{-1}+W_c^\top W_c
		    \end{bmatrix}^{-1}.
		\end{align}
		Substituting into Eq. \eqref{eq:sigmaxinv},
		\begin{multline}
		    \Sigma_X^{-1}
		    = \sigma_v^{-2} \Big ( \rm{I}-W_t(\sigma_v^2\Sigma_t^{-1}+W_t^\top W_t)^{-1}W_t^\top \\
		    - W_c(\sigma_v^2\Sigma_c^{-1}+W_c^\top W_c)^{-1}W_c^\top
		    \Big).
		\end{multline}
		Substituting into Eq. \eqref{eq:Kcbar} and using $W_t^\top W_c=\bm{0}$, we obtain
		\begin{align}
		    \K_{\cbar}&=\Sigma_t W_t^\top\Big(\sigma_v^{-2}\big[\rm{I}-W_t(\sigma_v^2\Sigma_t^{-1}+W_t^\top W_t)^{-1}W_t^\top\big]\Big)\\
		    &=\Sigma_t W_t^\top \big(\Sigma_v + W_t\Sigma_t W_t^\top\big)^{-1}\\
		    &=\K_{\calib}.
		\end{align}
		The equivalence between the first and second lines can be obtained by the Woodbury matrix identity.
	\end{proof}

	The first line of the proof can be manipulated to show that both $\K_{\cbar}$ and $\K_{\calib}$ can be written as
	\begin{equation}
	    \sigma_v^{-2} \Sigma_t \left[ \rm{I} - W_t^\top W_t
	    ( \sigma_v^2 \Sigma_t^{-1} + W_t^\top W_t )^{-1} \right] W_t^\top.
	\end{equation}
	If $W_t^\top W_c = \bm{0}$, then $\K_{\cbar} W_c = \K_{\calib} W_c = \bm{0}$. In other words, the ideal estimators are insensitive to the subject dependent changes in appearance.

	\begin{theorem}
		Let $\Dtrain$ be an arbitrary training set with $\Itrain$ subjects and $\Jtrain$ samples per subject. If $W_t^\top W_c=\bm{0}$ and $\Sigma_v=\sigma_v^2 \rm{I}_k$, then
		\begin{align}
			\lim_{\Jtrain\to\infty} \MSE_{\cbar} \big(\Khat_{\decbar}\big) \geq
			\lim_{\Jtrain\to\infty} \MSE_{\cbar} \big( \Khat_{\dec} \big).\nonumber
		\end{align}
		\label{theoOr2}
	\end{theorem}

	\begin{proof}
	Since $\hat{E}[A_{i,j}]=\hat{E}_i\big[\hat{E}_j[A_{i,j}]\big]$ using \eqref{eq:gazeModel1} and~\eqref{eq:gazeModel2}, we have by the law of large numbers that
	\begin{align}
	    \hat{E}[g_{i,j} X_{i,j}^\top] & \xrightarrow[\Jtrain\to\infty]{} \Sigma_t W_t^\top+\hat{E}_i[b_i c_i^\top]W_c^\top, \label{eq:Ehat1}\\
	    \hat{E}[X_{i,j} X_{i,j}^\top] &
	    \xrightarrow[\Jtrain\to\infty]{}
	    W_t\Sigma_t W_t^\top+W_c\hat{\Sigma}_c W_c^\top+\Sigma_v,
	\end{align}
	and that
	\begin{align}
	    \hat{E}_i\big[\hat{E}_j[g_{i,j}]\hat{E}_j[X_{i,j}^\top]\big] &
	    \xrightarrow[\Jtrain\to\infty]{}
	    \hat{E}_i[b_i c_i^\top]W_c^\top, \label{eq:limit3}\\
	    \hat{E}_i\big[\hat{E}_j[X_{i,j}]\hat{E}_j[X_{i,j}^\top]\big] &
	    \xrightarrow[\Jtrain\to\infty]{}
        W_c\hat{\Sigma}_c W_c^\top, \label{eq:Ehat4}
	\end{align}
	where $\hat{\Sigma}_c=\hat{E}_i[c_i c_i^\top]$.

	Substituting Eqs. \eqref{eq:Ehat1}-\eqref{eq:Ehat4} into Eqs.~\eqref{eq:Kdecbar}-\eqref{eq:Kdecbar2} and using the Woodbury matrix identify and $W_t^\top W_c=\bm{0}$,
	\begin{align}
	    \hat{\K}_{\decbar} &
	    \xrightarrow[\Jtrain\to\infty]{}
	    \big(\Sigma_t W_t^\top+\hat{E}_i[b_i c_i^\top]W_c^\top\big) \nonumber \\
	    & \qquad \times \big(W_t\Sigma_t W_t^\top+
	     W_c\hat{\Sigma}_c W_c^\top+\Sigma_v\big)^{-1} \\
	    & = \big(\Sigma_t W_t^\top+\hat{E}_i[b_i c_i^\top]W_c^\top\big) \sigma_v^{-2} \nonumber  \\
	    & \quad \times \Big( \rm{I}-W_t(\sigma_v^2\Sigma_t^{-1}+ W_t^\top W_t)^{-1}W_t^\top \nonumber \\
	    & \quad -W_c(\sigma_v^2\hat{\Sigma}_c^{-1}+W_c^\top W_c)^{-1}W_c^\top\Big)\\
	    & = \sigma_v^{-2}\Big( \Sigma_t W_t^\top + \hat{E}_i[b_i c_i^\top]W_c^\top \nonumber \\
	    & \quad - \Sigma_t W_t^\top W_t \big( \sigma_v^2\Sigma_t^{-1} + W_t^\top W_t\big)^{-1} W_t^\top \nonumber \\
	    & \quad - \hat{E}_i[b_i c_i^\top] W_c^\top W_c (\sigma_v^2\hat{\Sigma}_c^{-1}+W_c^\top W_c)^{-1}W_c^\top\Big)\\
	    &=\K_{\calib}+\Delta\hat{\K},
	\end{align}
	where
	\begin{multline}
	    \Delta\hat{\K} = \hat{E}_i[b_i c_i^\top]W_c^\top  \\
	     - \hat{E}_i[b_i c_i^\top] W_c^\top W_c (\sigma_v^2\hat{\Sigma}_c^{-1}+W_c^\top W_c)^{-1}W_c^\top.
	\end{multline}

	Substituting Eqs. \eqref{eq:Ehat1}-\eqref{eq:Ehat4} into Eqs.~\eqref{eq:Kdec}, \eqref{eq:Kdec2} and~\eqref{eq:Kdec3}, we obtain
	\begin{equation}
	    \hat{\K}_{\dec} \xrightarrow[\Jtrain\to\infty]{} \K_{\calib}.
	\end{equation}

	The limiting values of $\hat{\K}_{\dec}$ and $\hat{\K}_{\decbar}$ differ by $\Delta\hat{\K}$, which goes to zero as $\Jtrain\to\infty$, since $b_i$ and $c_i$ are independent with zero-mean. Note that $\hat{\K}$ only depends on the training set and is a constant when during testing. Thus,

	\begin{equation}
	\begin{aligned}
	    \lim_{\Jtrain\to\infty}&\MSE_{\cbar}\big(\hat{\K}_{\decbar}\big)-\lim_{\Jtrain\to\infty}\MSE_{\cbar}\big(\hat{\K}_{\dec}\big)\\
	    &=\E\Big[||g_{i,j}-\K_{\calib}X_{i,j}-\Delta\hat{\K}X_{i,j}||_2^2\\
	    &\qquad\quad-||g_{i,j}-\K_{\calib}X_{i,j}||_2^2\Big]\\
	    &=\Tr\Big\{\E\Big[-\big(g_{i,j}-\K_{\calib}X_{i,j}\big)X_{i,j}^\top\Delta\Khat^\top\\
	    &\qquad\qquad\quad -\Delta\Khat^\top X_{i,j}\big(g_{i,j}-\K_{\calib}X_{i,j}\big)^\top\\
	    &\qquad\qquad\quad +\Delta\Khat X_{i,j} X_{i,j}^\top \Delta\Khat\Big]\Big\}\\
	    &=\Tr\big\{\Delta\Khat \Sigma_X \Delta\Khat^\top\big\},
	\end{aligned}
	\end{equation}
	where the last equality holds by noting that $\K_{\calib}$ is the optimal estimator of $g_{i,j}$ for $X_{i,j}$, which implies that the estimation error is orthogonal to the data $X_{i,j}$. Since the autocorrelation matrix $\Sigma_X$ is positive semidefinite, $\Delta\Khat \Sigma_X^\top \Delta\Khat^\top\geq0$, which completes the proof.

	\end{proof}

    Referring to the proof, we can see that $\Khat_{\dec} \xrightarrow[\Jtrain \to \infty]{} \K_{\calib}$ whether or not $W_t^\top W_c = \bm{0}$. Since $\K_{\calib}$ minimizes $\MSEC$, it is always true that
	\begin{equation}
		\lim_{\Jtrain\to\infty} \MSE_{\calib} \big( \Khat_{\dec} \big)
		\leq
		\lim_{\Jtrain\to\infty} \MSE_{\calib} \big(\Khat_{\decbar}\big).
		\label{eq:MSEcal_ineq}
	\end{equation}

	\subsection{Simulation}

	To examine the performance $W_t^\top W_c \neq \bm{0}$, we simulated the model in Eqs.~\eqref{eq:gazeModel1} and~\eqref{eq:gazeModel2} with $s=100$, $p=10$, $\Sigma_c = 36\rm{I}$, $\Sigma_v = 4\rm{I}$ and $\Sigma_b = 4\rm{I}$, but with $t_{i,j}\sim\mathcal{U}(-30, 30)$ to better match the actual distribution of gaze in our other datasets.

	We ran simulations for four different values (0, 0.2, 0.4 and 0.6) of  $\| W_t^\top W_c \|$ where we used the max norm
	$\| A \| = \max\limits_{i,j} | a_{i,j} |$. For all four cases, the matrix $W_t$ was the same: a randomly generated matrix whose two columns were of unit length. For the case $\| W_t^\top W_c \| = 0$, the matrix $W_c$ was randomly generated with $p$ unit length columns orthogonal to the columns of $W_t$. For the other values of $\| W_t^\top W_c \|$, the matrix $W_c$ was generated by taking the original $W_c$ matrix and randomly adding or subtracting one of the columns from $W_t$ scaled by the the desired value of $\| W_t^\top W_c \|$.

	For each value of $\| W_t^\top W_c \|$, we examined the performance of the four different combinations of training ($\decbar, \dec$) and testing ($\cbar, \calib$), as $\Itrain$ varied. In each case, we ran 500 trials with $\Jtrain=50$ randomly generated image-gaze pairs per subject. For testing, we used the same test set with 300 subjects and 1,000 images per subject.

	Fig.~\ref{fig:LinearSysRes} plots the simulation results. Consistent with Theorem~\ref{theoOr1}, when $\| W_t^\top W_c \| =0$ (Fig.~\ref{fig:LinearSysRes}(a)), the errors when training without and with gaze decomposition converge to the same value as $\Itrain$ increases (compare the red and blue curves), whether or not calibration is used in testing (for both the solid and dashed curves). Testing with calibration is better than testing without calibration (compare the solid and dashed curves). Consistent with Theorem~\ref{theoOr2}, when tested without calibration (solid curves), the performance of the estimator trained with gaze decomposition (blue) is better than that of the estimator trained without (red).

	Consistent with Eq.~\eqref{eq:MSEcal_ineq}, which holds for all values of $\| W_t^\top W_c \|$, when tested with calibration (dashed curves), the performance of the estimator trained with gaze decomposition (blue) is better than that of the estimator trained without (red), independent of the value of $\Itrain$. As $\| W_t^\top W_c \|$ increases, the advantage of using the gaze decomposition increases (the dashed blue and red curves separate by more and more).

	Consistent with Eq.~\eqref{eq:inequality1}, when testing without calibration (solid curves) and for large values of $\Itrain$, the estimator trained without gaze decomposition (red) eventually performs better than the estimator trained with calibration (blue), and the performance mismatch increases with $\| W_t^\top W_c \|$. However, if the value of $\Itrain$ is small enough, the estimator with gaze decomposition performs better. The value of $\Itrain$ at which the two curves cross over increases as $\| W_t^\top W_c \|$ decreases. According to Theorem~\ref{theoOr2}, when $\| W_t^\top W_c \| =0$, the crossover is when $\Itrain = \infty$.
	\begin{figure}[!t]
		\centering
		\subfloat[$\| W_t^\top W_c \|=0$]{
			\includegraphics[width=4.5cm]{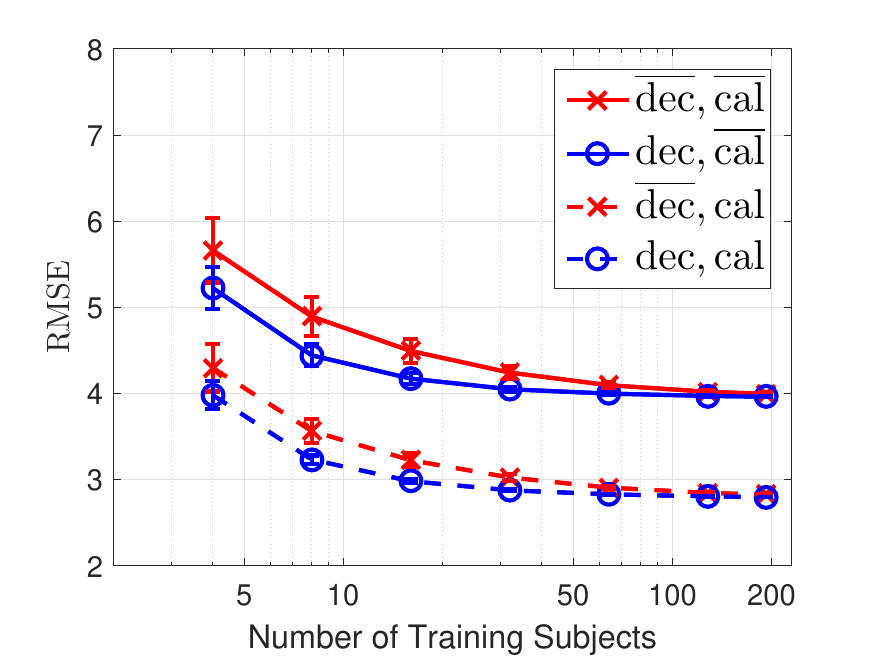}
		}
		\subfloat[$\| W_t^\top W_c \|=0.2$]{
			\includegraphics[width=4.5cm]{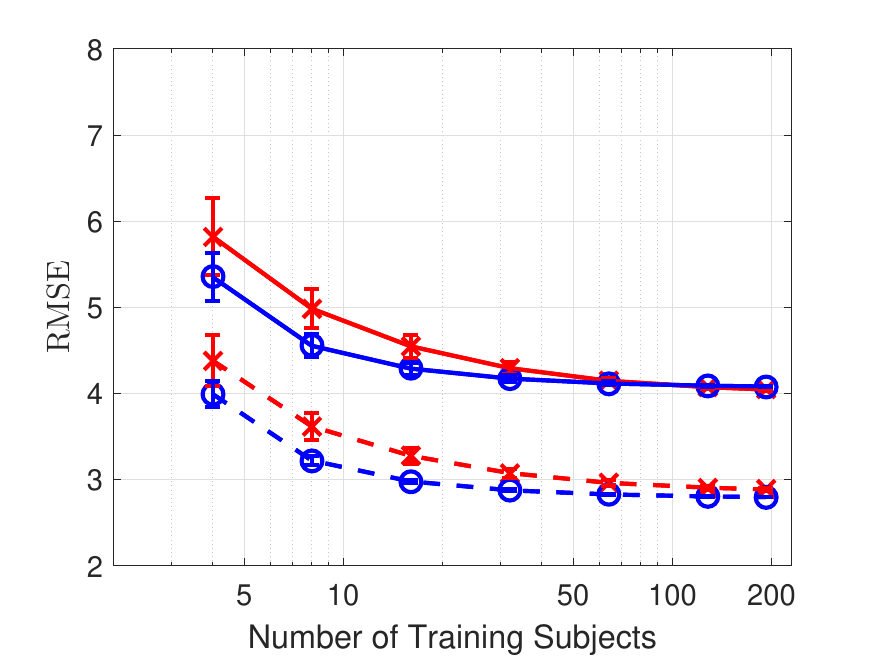}
		}\\
		\subfloat[$\| W_t^\top W_c \|=0.4$]{
			\includegraphics[width=4.5cm]{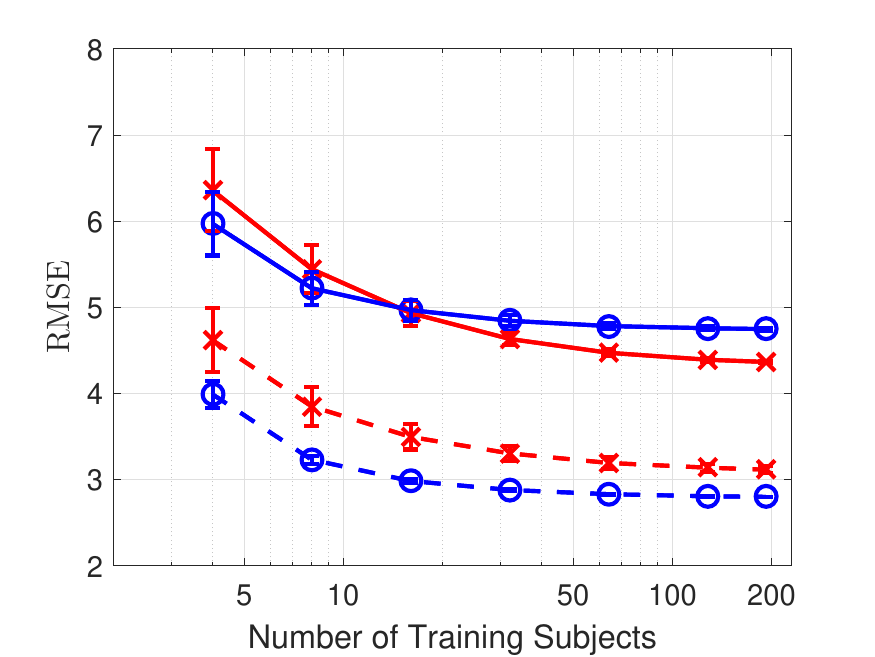}
		}
		\subfloat[$\| W_t^\top W_c \|=0.6$]{
			\includegraphics[width=4.5cm]{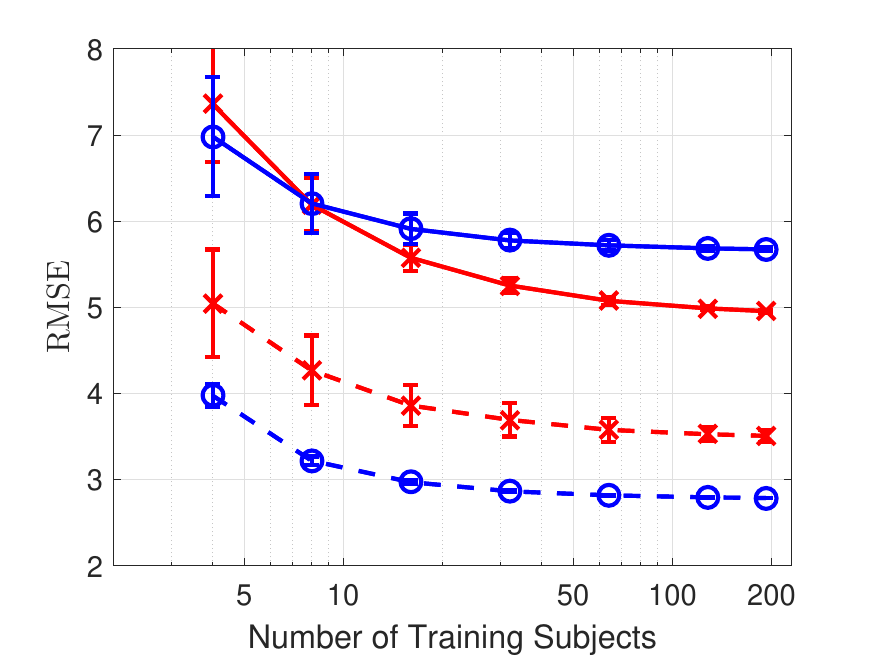}
		}
		\caption{Simulated testing results of $\hat{\K}$ and $\hatKdec$ under different values of $\| W_t^\top W_c \|$. Solid (dashed) lines indicated testing without (with) calibration. Red lines with $\times$ markers (blue lines with $\circ$ markers) indicate training without (with) gaze decomposition. RMSE means Root Mean Square Error. Error bars indicate standard deviations over 500 trials.}
		\label{fig:LinearSysRes}
	\end{figure}

	\ifCLASSOPTIONcaptionsoff
	\newpage
	\fi

	\bibliographystyle{IEEEtran}
	\bibliography{myreference}